\documentclass[10pt,a4paper]{article}
\usepackage[latin1]{inputenc}
\usepackage[latin1]{inputenc}
\usepackage{amsmath}
\usepackage{amsthm}
\usepackage{amsfonts}
\usepackage{amssymb}
\usepackage{graphicx}
\usepackage{subfig}
\usepackage{tikz}
\usetikzlibrary{decorations.pathreplacing}
\usepackage{enumerate}  
\usepackage{stmaryrd}
\usepackage{color}
\usepackage{dirtytalk}
\usepackage{mathrsfs}
\usepackage{marginnote}
\usepackage{fullpage}
\usepackage[ruled, lined, linesnumbered, longend]{algorithm2e}
\newtheorem*{Theorem*}{Theorem}
\newtheorem{Theorem}{Theorem}[section]
\newtheorem{Lemma}[Theorem]{Lemma}

\newtheorem{Proposition}[Theorem]{Proposition}

\theoremstyle{definition}
\newtheorem{Definition}{Definition}

\theoremstyle{remark}
\newtheorem{Remark}[Theorem]{Remark}

\newcommand*{\as}{\mathsf{AffineSpan}}
\newcommand*{\ls}{\mathsf{LinearSpan}}
\newcommand*{\ad}{\mathsf{AffineDim}}
\newcommand*{\ld}{\mathsf{LinearDim}}

\newcommand*{\cone}{\mathsf{Cone}}
\newcommand{\leftm}{\{\!\!\{}
\newcommand{\rightm}{\}\!\!\}}
\newcommand{\bleftm}{\bigl\{\!\!\bigl\{}
\newcommand{\brightm}{\bigl\}\!\!\bigl\}}

\newcommand*{\defin}[1]{\textbf{\emph{#1}}}

\title{Three iterations of $(d-1)$-WL test distinguish non-isometric clouds of $d$-dimensional points}
\author{Valentino Delle Rose$^{5}$ \and Alexander Kozachinskiy$^{1}$ \and Crist\'obal Rojas$^{1,2}$ \and Mircea Petrache$^{1,2,4}$ 
\and Pablo Barcel\'o$^{1,2,3}$}

\date{\small %
    $^1$Centro Nacional de Inteligencia Artificial, Chile\\%
    $^2$Instituto de Ingenier\'ia Matem\'atica y Computacional, Universidad Cat\'olica de Chile\\
    $^3$Instituto Milenio Fundamentos de los Datos, Chile\\
    $^4$Departamento de Matem\'atica, Universidad Cat\'olica de Chile\\
	$^5$ANVUR, Rome, Italy}

\begin{document}

\maketitle

\begin{abstract}
 The Weisfeiler--Lehman (WL) test is a fundamental iterative algorithm for checking isomorphism of graphs. It has been observed that it also underlies the design of several graph neural network architectures, whose capabilities and performance can be understood in terms of the expressive power of this test. Motivated by recent developments in machine learning applications to datasets involving geometric objects, we study the expressive power of a geometric version of the WL test of order $\ell$ (referred to as geometric $\ell$-WL). More precisely, we consider clouds of points in euclidean $d$-dimensional space represented by complete distance graphs, and ask whether 
the $\ell$-WL test is {\em complete} in this setting. That is, whether it can distinguish, up to isometry, any arbitrary such cloud. 
 Our main result states that for any $d\geq 2$, the $(d-1)$-WL test is complete for clouds of points in $d$-dimensional Euclidean space, and that only three iterations of the test suffice. We also observe that the $d$-WL test only requires one iteration to achieve completeness for clouds in dimension $d$. 
Our paper thus provides a complete understanding of the 3-dimensional case: it was shown in previous works that 1-WL is not complete in $\mathbb{R}^3$, and from our results it follows that 2-WL is complete there. We also strengthen the incompleteness of 1-WL in three dimensions by showing that it is even unable to recognize whether a given three-dimensional point cloud is planar. That is, there exist planar point clouds in $\mathbb{R}^3$ that 1-WL cannot distinguish from some non-planar ones. Finally, we show that 2-WL is not complete in $\mathbb{R}^6$, leaving as an open question its completeness in $\mathbb{R}^{d}$ for $d = 4,5$. 
\end{abstract}

\newpage
\section{Introduction}

{\bf Point Clouds in Machine Learning.}
Recent work in machine learning has raised the need to develop effective and efficient 
tests for checking if two three-dimensional point clouds, i.e., finite sets of points in $\mathbb{R}^3$, are {\em isometric} \cite{pozdnyakov2022incompleteness,3d-case,https://doi.org/10.48550/arxiv.2206.07697,https://doi.org/10.48550/arxiv.2302.05743}. 
Recall that, given two such point clouds $P$ and $Q$,
an isometry is a distance-preserving bijection between the points in $P$ and $Q$.  
The importance of these tests is 
that they provide the foundations for designing neural network architectures on point clouds that are capable of fully exploiting the 
structure of the data \cite{DBLP:conf/iclr/XuHLJ19,DBLP:conf/aaai/0001RFHLRG19}. It has been observed that the {\em incompleteness} of any such an architecture, i.e., the inability to recognize a 
point cloud up to isometry, can affect its learning performance \cite{PhysRevLett.125.166001}.

Intuitively, the general idea is as follows. Given a point cloud $P$, one starts by computing an \emph{invariant} of $P$, that is, a mathematical object capturing certain information about $P$, which must be the same for isometric clouds. To "test" whether two point clouds are isometric, one then simply compares their computed invariants. If they differ, then the clouds are not isometric. Such an invariant is \emph{complete} when also the converse holds: only isometric clouds can have the same invariant. In other words, we obtain a \emph{complete test} precisely when the information captured by the invariant is sufficient to reconstruct a given point cloud up to isometry.

For example, the simplest interesting invariant for a given cloud point $P$ consists of the multiset of all distances between points in $P$ \cite{boutin2004reconstructing, boutin2007point} (see also \cite{brinkman2012invariant, memoli2022distance} for some more recent extensions). In this case, it is known that reconstruction is not possible for some clouds $P \subset  \mathbb R^d$, with $d\geq 2$. It can be shown, however, that the collection of clouds with $N$ points that are "problematic" (those that cannot be reconstructed) is contained in a set of co-dimension 1 within the space of all such clouds and thus has zero measure there \cite{boutin2004reconstructing}.

For many Machine Learning applications, however, the existence of these exceptional, non-reconstructible point clouds may be  problematic. In the setting of computational chemistry applications for instance \cite{pozdnyakov2022incompleteness}, machine learning setups for geometric learning that fulfill smoothness requirements (usually imposed for robustness) may be harder to train to high accuracy on a given dataset if they are based on an incomplete test, even if the "exceptional" non-reconstructable cases are not present in the dataset. 
Thus the question of finding a minimal invariant guaranteed to  distinguish non-isometric point clouds without exceptions is both interesting from a purely mathematical point of view, and relevant in practice.

\medskip
{\bf The geometric WL test.} 
Point clouds can be represented as complete graphs in which each edge is labeled with the distance between the corresponding points. Under this representation, detecting the isometry of two point clouds is reduced to detecting an isomorphism between their graph representations. Not surprisingly, much of the recent work on developing so-called {\em geometric} tests for detecting isometries over point clouds is inspired by the 
literature on isomorphism tests from graph theory. Of particular importance in this context 
has been the use of geometric versions of the 
well-known {\em Weisfeiler-Lehman test} (WL test) for graph isomorphism \cite{WL}.

Intuitively, the $\ell$-order geometric WL test ($\ell$-WL test), for $\ell \geq 1$, iteratively colors each tuple $\bar v$ of $\ell$ points in a point cloud. The color of 
$\bar v$ in round 0 is a complete description of the mutual distances between the points that belong to the tuple. 
In round $t+1$, for $t \geq 0$, the color of $\bar v$ is updated 
by combining in a suitable manner its color in iteration $t$ with 
the color of each one of its {\em neighbors}, i.e., the tuples $\bar v'$ that are obtained from $\bar v$ 
by exchanging exactly one component of $\bar v$ with another point in the cloud. The order of the WL test is therefore a measure of 
its computational cost: the higher the $\ell$, the more costly it is to implement the $\ell$-order WL test.  

For checking if two point clouds are isometric, the geometric WL test compares the resulting color patterns. If they differ, 
then we can be sure the point clouds are not isometric (that is, the test is {\em sound}). An important question, therefore, is what is the minimal $\ell \geq 1$ for which the geometric $\ell$-WL test is {\em complete}, i.e., the fact that the color patterns obtained in two point clouds are the same implies that they are isometric. Understanding which is the simplest complete invariant in this scenario is thus essential not only for developing ``complete" architectures but also to make them 
as efficient as possible in terms of the computational resources they need to use. 

There has been important progress on this problem recently: (a) Pozdnyakov and Ceriotti have shown that the geometric 1-WL test is incomplete for point clouds in 3D; 
that is, there exist isometric point clouds in three dimensions 
that cannot be distinguished by the geometric 1-WL test \cite{pozdnyakov2022incompleteness}; (b) Hordan et al.~have proved that 3-WL test is complete in 3D after 1 iteration when initialized with Gram matrices of the triples of points instead of the mutual distances in these triples. A similar result has recently been obtained in  \cite{https://doi.org/10.48550/arxiv.2302.05743}. Hordan et al.~also gave a complete ``2-WL-like'' test, but this test explicitly uses coordinates of the points.

\medskip

{\bf Our contribution.}  
 Our previous observation shows an evident gap in our understanding of the problem: What is the minimum $\ell$, where $\ell = 2,3$, for which the 
geometric $\ell$-dimensional WL test is complete over three-dimensional point clouds? 
Our main contributions are the following:

\begin{itemize} 
    \item We show that for any $d > 1$ the geometric $(d-1)$-WL test is complete for detecting isometries over point clouds in $\mathbb{R}^d$.  
    This is the positive counterpart of the result
    in \cite{DBLP:conf/aaai/0001RFHLRG19} (namely, that geometric 1-WL test is incomplete in dimension $d=3$), by showing that geometric 1-WL is complete for $d=2$ and that geometric 2-WL test is complete for $d=3$. Further, only three rounds of the geometric $(d-1)$-WL test suffice to obtain this completeness result. 
    \item We provide a simple proof 
that a single round of the geometric $d$-WL test is sufficient for identifying point clouds in $\mathbb{R}^d$ up to isometry, for each $d \ge 1$.
This can be seen as a refinement of a result from  \cite{3d-case}, with the difference that our test is initialized with the mutual 
distances inside $d$-tuples of points (as in the classical setting) while theirs is initialized with 
Gram matrices of $d$-tuples of points. In other words, the initial coloring of~\cite{3d-case}, for each $d$-tuple of points, in addition to their pairwise distances, includes their distances to the \emph{origin}, while in our result we do not require this additional information.

\item We strengthen previous lower bounds for the 1-WL test by giving an example of two non-isometric point clouds in $\mathbb{R}^3$, one of which is planar and the other is not, that are not distinguished by the 1-WL test. In other words, not only the 1-WL test is unable to reconstruct a point cloud up to an isometry in $\mathbb{R}^3$, it cannot even compute its dimensionality.

\item Finally, we show that 2-WL is not complete in $\mathbb{R}^6$. Our results thus narrow down the maximal dimension $d$ for which 2-WL is complete to $d\in\{3, 4, 5\}$. We leave as an open question to figure out this maximal $d$ exactly.
\end{itemize}

These results, as well as previously mentioned results obtained in the literature, are all based on the standard {\em folklore} version of the $\ell$-WL test (as defined, e.g., in \cite{DBLP:journals/combinatorica/CaiFI92}). This is important because another version of the test, known as the {\em oblivious} $\ell$-WL test, has also been studied in the machine learning literature \cite{DBLP:conf/aaai/0001RFHLRG19,DBLP:journals/corr/abs-2112-09992}. It is known that, for each $\ell \ge 1$, the folklore $\ell$-WL test has the same discriminating power as the oblivious $(\ell+1)$-WL test \cite{DBLP:conf/lics/Grohe21}.  

\medskip
{\bf Further related results.} As further related results, in \cite{joshi2023expressive} some geometric WL tests have been compared to the expressivity of invariant and equivariant graph neural networks. Non-geometric related results include e.g. \cite{furer2001weisfeiler}, where for explicit graphs on $n$ vertices it is shown that $\ell$-WL requires $O(n)$ iterations for distinguishing them, whereas $2\ell$-WL requires $O(\sqrt n)$ and $(3\ell-1)$-WL requires $O(\log n)$ iterations. In \cite{zhang2023rethinking}, the authors propose and study  generalized distances on non-geometric graphs based on biconnectivity, which appears as a promising notion.

Since the extended abstract of this paper came out \cite{NEURIPS2023_1e6cf8f7}, a graph neural network (GNN) architecture mimicking the geometric 2-WL test in $\mathbb{R}^3$, and provably equivalent to it in distinguishing power (already for 1-dimensional features, as long as a non-polynomial analytic activation is used), was suggested in \cite{DBLP:conf/icml/HordanAD24}. Due to our result, this GNN architecture is a complete isomorphism test in $\mathbb{R}^3$ already after 3 iterations. In turn, the authors of \cite{sverdlov2024expressive} have investigated more standard message-passing neural networks (MPNNs), well-known to be equivalent to the 1-WL test \cite{DBLP:conf/aaai/0001RFHLRG19}, for sparse representations of point clouds in $\mathbb{R}^3$ (where instead of having a complete distance graph, every point is connected to just a few nearest points). It was shown there MPPNs manage to distinguish most of non-isometric point clouds in this setting.

\section{Formal Statement of the Main Result} 
\vspace{-0.2cm}

For a finite set $A$ and for function $f$ whose domain is $A$, by $\leftm f(a) \mid a\in A\rightm$ we denote a \emph{multiset} whose support is the image of $f$, and the multiplicity of an element $b\in f(A)$ is $|f^{-1}(b)|$.

Consider a cloud of $n$ points $S=\{p_1, \dots, p_n\}$ in $\mathbb{R}^d$. We are interested in the problem of finding representations of such clouds that completely characterize them up to isometries, while at the same time being efficient from an algorithmic point of view. Our main motivation is to understand the expressiveness of the WL algorithm when applied to point clouds in euclidean space seen as complete distance graphs. Let us briefly recall how this algorithm works.

A function whose domain is $S^\ell$ will be called an \emph{$\ell$-coloring} of $S$. The $\ell$-WL algorithm is an iterative procedure which acts on $S$ by assigning, at iteration $i$, an $\ell$-coloring $\chi_{\ell,S}^{(i)}$ of $S$. 

{\bf Initial coloring.} 
The initial coloring, $\chi_{\ell,S}^{(0)}$, assigns to each $\ell$-tuple $\mathbf{x} = (x_1, \ldots, x_\ell)\in S^\ell$ the color $\chi_{\ell,S}^{(0)}(\mathbf{x})$ given by the $\ell\times \ell$ matrix 
 \vspace{-0.2cm}
\[
\chi_{\ell,S}^{(0)}(\mathbf{x})_{ij} = d(x_i, x_j) \quad i,j = 1,\dots,\ell
\]
of the relative distances inside the $\ell$-tuple (for $\ell=1$ we have a trivial coloring by the $0$ matrix). 

{\bf Iterative coloring.} 
 At each iteration, the \emph{$\ell$-WL algorithm} updates the current coloring $\chi_{\ell,S}^{(i)}$ to obtain a refined coloring $\chi_{\ell,S}^{(i+1)}$. The update operation is defined slightly differently depending on whether $\ell = 1$ or $\ell\ge 2$. 
 \vspace{-0.1cm}
\begin{itemize} 
\item 
 For $\ell = 1$, we have:
\[ \chi_{1, S}^{(i +1)}(x) = \Big(\chi_{1,S}^{(i)}(x), \leftm (d(x, y), \chi_{1,S}^{(i)}(y)) \mid y\in S\rightm\Big).\]
That is, first, $\chi_{i+1}(x)$ remembers the coloring of $x$ from the previous step. Then it goes through all points $y\in S$. For each $y$, it stores the distance from $x$ to $y$ and also the coloring of $y$ from the previous step, and it remembers the multiset of these pairs. Note that one can determine which of these pairs comes from $y = x$ itself since this is the only point with $d(x, y) = 0$. We also note that $\chi_{1,S}^{(1)}(x)$ corresponds to the multiset of distances from $x$ to the points of $S$. 
\item 
To define the update operation for $\ell \ge 2$, we first introduce additional notation. Let $\mathbf{x} = (x_1, \ldots, x_\ell) \in S^\ell$ and $y\in S$. By $\mathbf{x}[y/i]$ we mean the tuple obtained from $\mathbf{x}$ by replacing its $i$th coordinate by $y$. Then the update operation can be defined as follows:
 \vspace{-0.1cm}
\begin{equation}
    \chi^{(i+1)}_{\ell,S}(\mathbf{x}) = \Big(\chi^{(i)}_{\ell,S}(\mathbf{x}), \leftm \left(\chi^{(i)}_{\ell,S}(\mathbf{x}[y/1]),\ldots, \chi^{(i)}_{\ell,S}(\mathbf{x}[y/\ell])\right)
    \mid y\in S\rightm\Big).
\end{equation}

In other words, first, $ \chi^{(i+1)}_{\ell,S}(\mathbf{x})$ remembers the coloring of $\mathbf{x}$ from the previous step, as before. Then, it goes  through all $y\in S$ and considers the $\ell$ tuples $\mathbf{x}[y/1], \ldots, \mathbf{x}[y/\ell]$. It then takes the colorings of these tuples from the previous step and puts them into a tuple. The new coloring now remembers the multiset of these tuples. 
\end{itemize} 
 \vspace{-0.1cm}
In this paper we show that the coloring obtained after 3 iterations of $(d-1)$-WL is a complete isometry invariant for point clouds in $\mathbb{R}^{d}$. More precisely, we show:

\begin{Theorem}[Main Theorem]
\label{main}
    For any $d\ge 2$ and for any finite set $S\subseteq \mathbb{R}^{d}$, the following holds. Let $\chi_{d-1,S}^{(3)}$ be the coloring of $S^{d-1}$  obtained after 3 iterations of the $(d-1)$-WL algorithm  on the distance graph of $S$. Then, knowing the multiset 
    \[\mathcal{M}_{d-1}^{(3)}(S)=\leftm \chi_{d-1,S}^{(3)}(\mathbf{s})\mid \mathbf{s}\in S^{d-1} \rightm,\] one can determine $S$ up to an isometry.
\end{Theorem}

Our proof is constructive in the sense that we exhibit an algorithm which, upon input $\mathcal{M}_{d-1}^{(3)}(S)$, computes the coordinates of a point cloud $S'$ which is isometric to $S$. In particular, if $\widetilde{S}\subset \mathbb{R}^d$ is another point cloud such that $\mathcal{M}_{d-1}^{(3)}(\widetilde{S})=\mathcal{M}_{d-1}^{(3)}(S)$, then $S$ and $\widetilde{S}$ are isometric.

\medskip

{\bf Organization of the paper.}
In Section \ref{sec:2d} we give a proof of Theorem \ref{main} for $d = 2$, which while being somewhat simpler, will allow us to introduce the general strategy. We give a proof of Theorem \ref{main} in general case in Section \ref{sec:3d}.
In Section \ref{sec:one_round}, we establish a weaker version of Theorem \ref{main}, namely, that the $d$-WL test is complete in $\mathbb{R}^d$ for any $d$. However, this result is of independent interest because we show that just 1 iteration suffices in this case.
We finish the paper with our negative results. In Section \ref{sec:1wlcounterexample}, we give an example of two points clouds in $\mathbb{R}^3$, one of which is planar and the other is not, that are not distinguished by the 1-WL test. In Section \ref{sec:2wl}, we show that the 2-WL test is not complete in $\mathbb{R}^6$.

\section{Three iterations of 1-WL distinguish clouds in the plane}
\label{sec:2d}

Let $S\subseteq \mathbb{R}^2$ be a cloud of $n$ points in the plane. Our task is to reconstruct $S$ up to an isometry, using as input the information contained in $\chi_{1,S}^{(3)}$. This means to find a point cloud $S^\prime$ in the plane which is an image of $S$ under some isometry. Our proof has two main steps: \emph{Initialization} and \emph{Reconstruction}. In the Initialization Step we show how to extract from $\chi_{1,S}^{(3)}$ the relevant information we need, which we call \emph{initialization data}. In the Reconstruction step, we describe an algorithm that, given some initialization data, computes the coordinates of the desired isometric cloud. 


In our reconstruction algorithm, we employ the notion of the \emph{barycenter} of a point cloud (also known as the center of mass), which we denote by $b$, and is  defined by: 
\[
    b=\frac{1}{n}\sum_{w \in S} w.
\]
For simplicity, we translate $S$ by $-b$ so that the new barycenter sits at $b = 0$. Notice that since our reconstruction is up to isometries, this assumption does not affect the generality of our result. For each $w\in S$, let $\|w\|$ denote its norm (its distance to $b=0$).

\medskip
We say that two points $u, v\in S$ satisfy the \emph{cone condition} if $u\neq 0$, $v\neq 0$, and, moreover,
\begin{itemize}
    \item if $0, u, v$ lie on the same line, then all points of $S$ lie on this line;
    \item if $0, u, v$ do not lie on the same line, then the interior of 

$\cone(u,v)=\{\alpha u + \beta v: \alpha,\beta \in [0,+\infty)\}$
does not contain points from $S$ (see the first picture on Figure \ref{fig:main}, the red area between $(0, u)$ and $(0, v)$ is disjoint from $S$).
\end{itemize}

\medskip
In order to initialize our reconstruction algorithm, we need the following information about $S$. We assume that $S$ has more than 1 point  (otherwise there is nothing to do).

\medskip
\noindent{\textbf{\emph{Initialization Data}}: the initialization data consists of a real number $d_0\geq 0$ and two multisets $M, M'$ such that for some $u, v\in S$, satisfying the cone condition, it holds that $d_0 = d(u, v)$ and
\[
M=M_u = \bleftm (d(u,y),\|y\|) : y \in S \brightm  \text{;} \quad M'=M_v = \bleftm (d(v,y),\|y\|) : y \in S \brightm.
\]

We will start by describing the Reconstruction Algorithm, assuming that the initialization data is given. We will then show how to extract this data from $\chi_{1,S}^{(3)}$ in the Initialization Step bellow. 
\medskip

\noindent\textbf{\emph{Reconstruction algorithm.}} 
 Assume that initialization data $(d_0, M,M')$ is given. Our task is to determine $S$ up to isometry. 
 Note that from $M$ we can determine $\|u\|$. Indeed, in $M$ there exists exactly one element whose first coordinate is $0$, and this element is $(0, \|u\|)$. Likewise, from $M^\prime$ we can determine $\|v\|$. We are also given $d_0 = d(u, v)$. Overall, we have all the distances between $0$, $u$, and $v$. Up to a rotation of $S$, there is only one way to put $u$, it has to be somewhere on the circle of radius $\|u\|$, centered at the origin. We fix any point of this circle as $u$. After that, there are at most two points where we can have $v$. More specifically, $v$ belongs to the intersection of two circles: one of radius $\|v\|$ centered at the origin, and the other of radius $d(u, v)$ centered at $u$. These two circles are different (remember that the cone condition includes a requirement that $u\neq 0$). Hence, they intersect by at most two points. These points are symmetric w.r.t.~the line that connects the centers of the two circles, i.e.,  $0$ and $u$. Thus, up to a reflection through this line (which preserves the origin and $u$), we know where to put $v$.
 
 Henceforth, we can assume the coordinates of $u$ and $v$ are known to us. Note that so far, we have applied to $S$ a translation (to put the barycenter at the origin), a rotation (to fix $u$), and a reflection (to fix $v$). We claim that, up to this isometry, $S$ can be determined uniquely.

 Let $\mathsf{Refl}_u$ and $\mathsf{Refl}_v$ denote the reflections through the lines $(0, u)$ and $(0, v)$, respectively. We first observe that from $M$ we can restore each point of $S$ up to a reflection through the line $(0, u)$. Likewise, from $M^\prime$ we can do the same with respect to the line $(0,v)$. More precisely, we can compute the following multisets:
 \[L_u = \bleftm \{y,\mathsf{Refl}_u y\} \mid y \in S\brightm, \qquad L_v =  \bleftm \{y,\mathsf{Refl}_v y\} \mid y \in S\brightm. \]

Indeed, consider any
   $(d(u, y), \|y\|)\in M$. What can we learn about $y\in S$ from this pair of numbers? These numbers are distances from $y$ to $u$ and to $0$. Thus, $y$ must belong to the intersection of two circles: one with the center at $u$ and radius $d(u, y)$ and the other with the center at $0$ and radius $\|y\|$. Again, since $u\neq 0$, these two circles are different. Thus, we obtain at most two points $z_1, z_2$ where one can put $y$. We will refer to these points as \emph{candidate locations} for $y$ w.r.t.~$u$. They can be obtained one from the other by the reflection $\mathsf{Refl}_u$
through the line connecting $0$ and $u$. Hence, $\{z_1,z_2\} = \{y,\mathsf{Refl}_u y\}$. To compute $L_u$, we go through $(d(u,y), \|y\|)\in M$, compute candidate locations $z_1, z_2$ for $y$, and put $\{z_1, z_2\}$ into $L_u$. In a similar fashion, one can compute $L_v$ from $M^\prime$.

Let us remark that elements of $L_u$ and $L_v$ are sets of size 2 or 1. A set of size 2 appears as an element of $L_u$ when some $y$ has two distinct candidate locations w.r.t.~$u$, that is when $y$ does not lie on the line $(0,u)$. In turn, when $y$ does lie on this line, we have $z_1 = z_2 = y$ for two of its candidate locations, giving us an element $\{y\}\in L_u$, determining $y$ uniquely. The same thing happens with respect to $L_v$ for points that lie on the line $(0, v)$,

The idea of our reconstruction algorithm is to gradually exclude some candidate locations so that more and more points get a unique possible location. What allows us to start is that $u$ and $v$ satisfy the cone condition; this condition gives us some area that is free of points from $S$ (thus, one can exclude candidates belonging to this area).

The easy case is when $0, u$, and $v$ belong to the same line. Then, by the cone condition, all points of $S$ belong to this line. In this case, every point of $S$ has just one candidate location. Hence, both multisets $L_u$ and $L_v$ uniquely determine $S$.

Assume now that $0, u$, and $v$ do not belong to the same line. As in the previous case, we can uniquely restore all points that belong to the line connecting $0$ and $u$, or to the line connecting $0$ and $v$ (although now these are two different lines). Indeed, these are points that have exactly one candidate location w.r.t.~$u$ or w.r.t.~$v$. They can be identified by going through $L_u$ and $L_v$ (we are interested in points $z$ with $\{z\}\in L_u\cup L_v$).

The pseudocode for our reconstruction algorithm is given in Algorithm \ref{alg:reconstruct}. We now give its verbal description.
Let us make a general remark about our algorithm. Once we find a unique location for some $y\in S$, we remove it from our set in order to reduce everything to the smaller set  $S\setminus\{y\}$. This is implemented by updating the multisets $L_u$ and $L_v$ so that $y$ is not taken into account in them. For that, we just remove $\{y, \mathsf{Refl}_u y\}$ from $L_u$ and $\{y, \mathsf{Refl}_v y\}$ from $L_v$ (more precisely, decrease their multiplicities by 1).

From now on, we assume that these two lines (connecting $0$ and $u$, and $0$ and $v$, respectively) are free of the points of $S$. These lines contain the border of the cone $\cone(u, v)$. At the same time, the interior of this cone is disjoint from $S$ due to the cone condition. Thus, in fact, the whole $\cone(u, v)$ is disjoint from $S$.

We now go through $L_u$ and $L_v$ in search of points for which one of the candidate locations (either w.r.t.~$u$ or w.r.t.~$v$) falls into the ``forbidden area'', that is, into $\cone(u, v)$. After restoring these points and deleting them, we notice that the   ``forbidden area'' becomes larger. Indeed, now in $S$ there are no points that fall into $\cone(u, v)$ under one of the reflections $\mathsf{Refl}_u$ or $\mathsf{Refl}_v$. In other words, the updated ``forbidden area'' is   $F=\cone(u,v)\cup\mathsf{Refl}_u\cone(u,v)\cup \mathsf{Refl}_v\cone(u,v)$. 
If the absolute angle between $u$ and $v$ is $\alpha_{uv}$, then, $F$ has total amplitude $3\alpha_{uv}$. We now iterate this process, updating $F$ successively. At each step, we know that after all removals made so far, $S$ does not have points in $F$. Thus, points of $S$ that fall into $F$ under $\mathsf{Refl}_u$ or under $\mathsf{Refl}_v$ can be restored uniquely. After deleting them, we repeat the same operation with $F \cup \mathsf{Refl}_u F \cup \mathsf{Refl}_v F$ in place of $F$.
\begin{algorithm}
\caption{Reconstruction algorithm}\label{alg:reconstruct}
$S$ := $\varnothing$\;
\For{$\{z\}\in L_u\cup L_v$}{
 \tcp{Restoring points from the lines $(0, u)$ and $(0,v)$}
Put $z$ into $S$\;
Remove $\{z\}$ from $L_u\cup L_v$\;
}
\medskip
$F$ := $\cone(u, v)$\;
\While{$F\neq \mathbb{R}^2$}
{
\For{$\{z_1, z_2\}\in L_u\cup L_v$}
{
\tcp{Restoring points that have one candidate location in the forbidden area}
\uIf{$z_1\in F$ or $z_2\in F$}
{
Set $y = z_1$ if $z_2 \in F$ and $y = z_2$ if $z_1\in F$\;
Put $y$ into $S$\;
 Remove $\{y,\mathsf{Refl}_u y\}$ from $L_u$ and $\{y,\mathsf{Refl}_v y\}$ from $L_v$\; 
}

}
\tcp{Updating forbidden area}
$F = F \cup \mathsf{Refl}_u F \cup \mathsf{Refl}_v F$\;
}
\medskip
Output $S$\;
\end{algorithm}

After $k$ such ``updates'', $F$ will consist of $2k+1$ adjacent angles, each of size $\alpha_{uv}$, with $\cone(u, v)$ being in the center. In each update, we replace $F$ with $F \cup \mathsf{Refl}_u F \cup \mathsf{Refl}_v F$. This results in adding two angles of size $\alpha_{uv}$ to both sides of $F$. Indeed, if we look at the ray $(0, u)$, it splits our current $F$ into two angles, one of size $(k+1)\alpha_{uv}$ and the other of size $k\alpha_{uv}$. Under $\mathsf{Refl}_u$, the part whose size is $(k+1)\alpha_{uv}$ adds an angle of size $\alpha_{uv}$ to the other part. Analogously, $\mathsf{Refl}_v$ adds an angle of size $\alpha_{uv}$ from the opposite side of $F$. See Figure \ref{fig:main} for the illustration of this process.

Within at most $1+\tfrac{\pi}{\alpha_{uv}}$ such steps, $F$ covers all angular directions, thus completing the reconstruction of $S$. 


\begin{figure}
\begin{minipage}{.5\linewidth}
\centering
\subfloat[]{\label{main:a}
\begin{tikzpicture}[scale=0.6]
   \node[inner sep=1pt,circle,draw,fill] (1) at (0,0) {};
   \node[draw=none] (b) at (-0.3,0) {$0$};

    \node[inner sep=1pt,circle,draw,fill] (2) at (2,0) {};
   \node[draw=none] (x) at (2.2,-0.2) {\Large$u$};

   \node[inner sep=1pt,circle,draw,fill] (3) at (3,1) {};
   \node[draw=none] (y) at (3.1,1.2) {\Large$v$};

\draw[ultra thick] (1)--(2);
\draw[ultra thick] (1)--(3);
\draw[fill=red,fill opacity=0.5]  (0,0) -- (4,0) -- (3.75,1.25) -- cycle;
\end{tikzpicture}
}
\end{minipage}%
\begin{minipage}{.5\linewidth}
\centering
\subfloat[]{\label{main:b}
\begin{tikzpicture}[scale=0.6]
  \node[inner sep=1pt,circle,draw,fill] (1) at (0,0) {};
   \node[draw=none] (b) at (-0.3,0) {$0$};

    \node[inner sep=1pt,circle,draw,fill] (2) at (2,0) {};
   \node[draw=none] (x) at (2.2,-0.2) {\Large$u$};

   \node[inner sep=1pt,circle,draw,fill] (3) at (3,1) {};
   \node[draw=none] (y) at (3.1,1.2) {\Large$v$};

\draw[ultra thick] (1)--(2);
\draw[ultra thick] (1)--(3);
   \draw[fill=red,fill opacity=0.5]  (0,0) -- (4,0) -- (3.75,1.25) -- cycle;
   \draw[fill=red,fill opacity=0.3]  (0,0) -- (4,0) -- (3.75,-1.25) -- cycle;
   \draw[fill=red,fill opacity=0.3]  (0,0) -- (3.75,1.25) -- (3.2,2.4) -- cycle;
   \end{tikzpicture}
   }
\end{minipage}\par\medskip
\centering
\subfloat[]{\label{main:c}\begin{tikzpicture}[scale=0.6]
  \node[inner sep=1pt,circle,draw,fill] (1) at (0,0) {};
   \node[draw=none] (b) at (-0.3,0) {$0$};

    \node[inner sep=1pt,circle,draw,fill] (2) at (2,0) {};
   \node[draw=none] (x) at (2.2,-0.2) {\Large$u$};

   \node[inner sep=1pt,circle,draw,fill] (3) at (3,1) {};
   \node[draw=none] (y) at (3.1,1.2) {\Large$v$};

\draw[ultra thick] (1)--(2);
\draw[ultra thick] (1)--(3);
   \draw[fill=red,fill opacity=0.5]  (0,0) -- (4,0) -- (3.75,1.25) -- cycle;
   \draw[fill=red,fill opacity=0.3]  (0,0) -- (4,0) -- (3.75,-1.25) -- cycle;
   \draw[fill=red,fill opacity=0.3]  (0,0) -- (3.75,1.25) -- (3.2,2.4) -- cycle;
 \draw[fill=red,fill opacity=0.1]  (0,0) -- (3.2,2.4) -- (2.27,3.28) -- cycle;
   \draw[fill=red,fill opacity=0.1]  (0,0) -- (3.75,-1.25) -- (3.2,-2.4) -- cycle;

   \draw [ultra thick,black,->] (3.94,-0.64) to[out=60,in=60]  (2.77,2.88);

    \draw [ultra thick,black,->] (3.54,1.85) to[out=-30,in=60] (5, 0) to[out=-120,in=30]   (3.54,-1.85);

   \end{tikzpicture}}

\caption{Growth of the ``forbidden'' area.}
\label{fig:main}
\end{figure}

\noindent\textbf{\emph{Initialization step.}} We explain how to obtain the Initialization Data about $S$ from $\mathcal{M}_1^{(3)}(S)$. 

We start by observing that from the first iteration of 1-WL, we can compute $\|x\|$ for all $x\in S$. As the following lemma shows, this holds in any dimension, with the same proof. We temporarily omit the current hypothesis $b=0$, in order to use the lemma later without this hypothesis.
\begin{Lemma}[The Barycenter Lemma] \label{l.orbits}
Take any $n$-point cloud $S\subseteq \mathbb{R}^{d}$ and let  
\[D_x = \leftm  d(x, y) \mid y\in S\rightm. \]
Then for every $x\in S$, knowing $D_x$ and the multiset $\leftm D_y\mid y\in S\rightm$, one can determine the distance from $x$ to the barycenter of $S$. 
\end{Lemma}
\begin{proof}
Consider the function $f\colon\mathbb{R}^{d}\to [0, +\infty)$ defined as
$f(x) = \sum_{y\in S} \| x - y\|^2$, 
namely $f(x)$ is the sum of the squares of all elements of $D_x$ (with multiplicities). It follows that $\sum_{y\in S}f(y)$ is determined by $\leftm D_y\mid y\in S\rightm$. The lemma is thus proved if we prove the following equality 
\begin{equation}\label{eq:barycalc}
    \| x - b\|^2=\frac1n\left(f(x)-\frac{1}{2n}\sum_{y\in S}f(y)\right).
\end{equation}
To prove the above, we first write:
    \begin{align*}
        f(x) = &\sum\limits_{y\in S} \| x - y\|^2 = \sum\limits_{y\in S} \| (x - b) + (b- y)\|^2 \\
        &=\sum\limits_{y\in S}\Big( \| x - b\|^2 +2\langle x - b, b- y\rangle + \| b-y\|^2 \Big) \\
        &= n\cdot  \| x - b\|^2 + 2\langle x - b, n\cdot b -  \sum\limits\limits_{y\in S} y\rangle + \sum\limits_{y\in S} \| b-y\|^2\\
        &=n\cdot  \| x - b\|^2 + \sum\limits_{y\in S} \| b-y\|^2 \quad \text{ (by definition of barycenter).} 
    \end{align*}
Denote $\Gamma = \sum_{y\in S} \| b-y\|^2$. Substituting the expression for $f(x)$ and $f(y)$ from above into the right-hand side of \eqref{eq:barycalc}, we get:  
\begin{align*}
    \frac1n\left(f(x)-\frac{1}{2n}\sum_{y\in S}f(y)\right) &= \frac1n\left(n\cdot  \| x - b\|^2 + \Gamma -\frac{1}{2n}\sum_{y\in S}\left(n\cdot  \| y - b\|^2 + \Gamma\right)\right) \\
&=\frac1n\left(n\cdot  \| x - b\|^2 + \Gamma - \frac{1}{2n} (n\Gamma + n\Gamma)\right) = \| x - b\|^2,
\end{align*}
as required.
\end{proof}

By definition, $\chi_{1,S}^{(1)}(x)$ determines the multiset $D_x = \bleftm d(x, y) \mid y\in S\brightm$ of distances from $x$ to points of $S$. 
Since we are given the multiset $\mathcal{M}_1^{(3)}(S)$, we also know the multset $\mathcal{M}_1^{(1)}(S) =\bleftm \chi_{1,S}^{(1)}(x) \mid x \in S\brightm$ (labels after the third iterations determine labels from previous iterations). In  particular, this gives us the multiset $\bleftm D_x\mid x\in S\brightm$. Overall, due to the Barycenter lemma, we conclude that $\chi_{1, S}^{(1)}(x)$ can be converted into $\|x\|$.

Now, remember that 
\[ \chi_{1, S}^{(2)}(x) = \Big(\chi_{1,S}^{(1)}(x), \leftm (d(x, y), \chi_{1,S}^{(1)}(y)) \mid y\in S\rightm\Big).\]
By converting $\chi_{1,S}^{(1)}(y)$ into $\|y\|$ for all $y\in S$ here, one can convert $\chi_{1,S}^{(2)}(x)$  into the multiset
\[M_x = \bleftm(d(x, y), \|y\|) \mid y\in S\brightm.\]
We need one more iteration to find $d(u, v), M_u, M_v$ for some $u, v\in S$ satisfying the cone condition. In fact, we only need
\[
   \chi_{1,S}^{(3)}(u)=\Bigl(\chi_{1,S}^{(2)}(u), \bleftm \bigl(d(u,y),\chi_{1,S}^{(2)}(y)\bigl):y\in S \brightm \Bigl)
\]
for arbitrary $u\in S$ with $u\neq 0$. Since $\chi_{1,S}^{(3)}(u)$ determines $\|u\|$, such $\chi_{1,S}^{(3)}(u)$  can indeed be selected from $\mathcal{M}_1^{(3)}(S)$ (and since we assume that $S$ has more than one point, we know that there are points in $S$ that are different from $0$).

Due to the fact that $\chi_{1,S}^{(2)}(y)$ can be converted to $M_y$, we can in turn convert $\chi_{1,S}^{(3)}(u)$ into the multiset
$\mathcal{A}=\bleftm \bigl(d(u,y),M_y\bigl):y\in S \brightm$. 
In particular, since $y=u$ is the only point for which $d(u,y)=0$, we can compute $M_u$ from $\mathcal{A}$. Once we have $M_u$, the rest of the initialization goes as follows. First note that for a given element $(d(u,y),M_y)$ in $\mathcal{A}$ with $d(u,y)>0$ (so that $y\neq u$), we can look in $M_y$ for the only element with $0$ as the first entry, whose second entry is then $\|y\|$. So we can obtain $(d(u,y),\|y\|)$. 
As in the Reconstruction Algorithm, we then have only two possibilities for the location of $y$ relative to $u$, say $y_1$ and $ y_2=\mathsf{Refl}_u(y_1)$. It follows that the absolute value of the angle $\alpha_{uy}$ between $u$ and $y$ is uniquely determined (if $\|y\| = 0$, we set $\alpha_{uy} = 0$), and we can compute it from $\mathcal{A}$. In order to select $v$, we go though $\mathcal{A}$ and look for $v$ such that $\alpha_{uv}$ is the smallest  angle among $\{\alpha_{uy}\mid y \in S, 0 <\alpha_{uy} < \pi\}$. If such a $v$ does not exist, all points of $S$ must lie on the line connecting $0$ and $u$. In this case, the cone condition is satisfied, for example, for $u$ and $v = u$. Thus, we can initialize with $d_0 = 0, M = M^\prime = M_u$.
 If $v$ as above exists, there can be no point in the interior of $\cone(u,v)$, since otherwise there would be $y$ with $0<\alpha_{uy}<\alpha_{uv} < \pi$, contradicting the minimality of $\alpha_{uv}$. Thus, the cone condition is satisfied for $u,v$. We can then set $d_0=d(u,v)$, $M=M_u$ and $M^\prime=M_v$. 

\section{Proof of Main Theorem for $d>2$}
\label{sec:3d}

We now present the proof for the case $d>2$. The strategy of the proof has the same structure as for $d=2$. Since the objects involved now are more general, it will be convenient to introduce some terminology. Let $\mathbf{x} = (x_1, \ldots, x_k)\in(\mathbb{R}^{d})^k$ be a $k$-tuple of points in $\mathbb{R}^{d}$. The \defin{distance matrix} of $\mathbf{x}$ is the $k\times k$ matrix $A$ given by $A_{ij} = d(x_i, x_j), i,j=1,\dots,k$.

Now, let  $S\subseteq \mathbb{R}^{d}$ be a finite set. Then the \defin{distance profile} of $\mathbf{x}$ w.r.t.~$S$ is the multiset 
\[
    D_{\mathbf{x}}=\leftm\big(d(x_1,y), \ldots, d(x_k,y)\big) \mid y \in S\rightm.
\]

As before, we let $b = \frac{1}{|S|}\sum_{y\in S}y$ denote the barycenter of $S$. For a finite set $G \subset \mathbb{R}^d$, we denote by $\ls(G)$ the linear space spanned by $G$, and by  $\as(G)$ the corresponding affine one. Their respective dimensions will be denoted by $\ld(G)$ and $\ad(G)$. 

 As for the case $d=2$, we start by distilling the Initialization Data required for the reconstruction, which is described relative to the barycenter $b$ of $S$. For convenience, we have decided not to assume at this stage 
that $S$ has been translated first to put $b$ at the origin, as we did for the sake of the exposition in the case $d=2$. This is now the task of the
isometry we apply to $S$ when choosing locations for its points, which we now completely relegate to the reconstruction phase. 

\begin{Definition} Let $S\in \mathbb{R}^d$ be a finite set and let $b$ be its barycenter. A $d$-tuple $\mathbf{x}=(x_1,\dots,x_d)\in S^d$  satisfies the \defin{cone condition} if 
\begin{itemize}
    \item $\ad(b,x_1,\dots,x_d)=\ad(S)$;
    \item if $\ad(S)=d$, then there is no $x\in S$ such that $x-b$ belongs to the interior of $\cone(x_1-b,\dots,x_d-b)$. 
\end{itemize}
\end{Definition}

\begin{Definition}
    For a tuple $\mathbf{x}=(x_1, \ldots, x_{d})\in S^{d}$, we define its \defin{enhanced profile} as
    \[
    EP(x_1, \ldots, x_{d}) = (A, M_1, \ldots, M_{d}),
    \]
    where $A$ is the distance matrix of the tuple $(b, x_1, \ldots, x_{d})$ and $M_i=D_{\mathbf{x}[b/i]}$ is the distance profile of the tuple $(x_1, \ldots, x_{i-1}, b, x_{i+1}, \ldots, x_{d})$ with respect to $S$. 
\end{Definition}
\begin{Definition} Let $S\in \mathbb{R}^d$ be a finite set and let $b$ be its barycenter. An \defin{initialization data} for $S$ is a tuple $(A,M_1,\dots,M_d)$ such that 
$(A,M_1,\dots,M_d) = EP(x_1, \ldots, x_d)$ for some $d$-tuple $\mathbf{x}=(x_1,\dots,x_d)\in S^d$ satisfying the cone condition. 
\end{Definition}

\begin{Remark}
The interested reader can verify that the Initialization Data condition extends the definition for $d=2$. Note that the first bullet of the Cone Condition is automatically verified for $d=2$, but is nontrivial for $d>2$.
\end{Remark}

The fact that an initialization data $(A,M_1,\dots,M_d)$ can be recovered from $\mathcal{M}_{d-1}(S)$ is ensured by the following proposition.

\begin{Proposition}[Initialization Lemma]\label{P.init}Take any $S\subseteq \mathbb{R}^{d}$. Then, knowing the multiset $\leftm \chi_{d-1,S}^{(3)}(\mathbf{s})\mid \mathbf{s}\in S^{d-1} \rightm$, one can determine an initialization data for $S$.
\end{Proposition}
\begin{proof}
We will need the following extension of the Barycenter Lemma. 
\begin{Lemma}
\label{c.radius}
Let $b$ be the barycenter of $S\subset \mathbb R^d, d>2$.
For any $\mathbf{x} = (x_1, \ldots, x_{d-1}) \in S^{d-1}$, knowing $\chi_{d-1,S}^{(1)}(\mathbf{x})$ and the multiset $\leftm \chi_{d-1,S}^{(1)}(\mathbf{y}) \mid \mathbf{y}\in S^{d-1}\rightm$, one can determine the tuple of distances $(d(x_1, b), \ldots, d(x_{d-1}, b))$.
\end{Lemma}
\begin{proof}
    Let $D_x$ be as in Lemma \ref{l.orbits}. We claim that we can determine $(D_{x_1},\dots,D_{x_{d-1}})$ and $\leftm D_y | y\in S\rightm$ from the information given in the lemma statement. By Lemma \ref{l.orbits}, this allows to determine $(d(x_1, b), \ldots, d(x_{d-1}, b))$.

    We have
    \[
        \chi^{(1)}_{d-1,S}(\mathbf{x}) = \Big(\chi^{(0)}_{d-1,S}(\mathbf{x}), \leftm \left(\chi^{(0)}_{d-1,S}(\mathbf{x}[y/1]),\ldots, \chi^{(0)}_{d-1,S}(\mathbf{x}[y/(d-1)])\right)
    \mid y\in S\rightm\Big).
    \]
    By definition, from  $\chi^{(0)}_{d-1,S}(\mathbf{x}[y/1])$ one can determine the tuple of distances $(d(x_2, y), \ldots, d(x_{d-1}, y))$. Hence, from the multiset $\leftm \chi^{(0)}_{d-1,S}(\mathbf{x}[y/1])\mid y\in S \rightm$ one can determine $(D_{x_2}, \ldots, D_{x_{d-1}})$. In turn, $D_{x_1}$ can be determined from, say, $\leftm \chi^{(0)}_{d-1,S}(\mathbf{x}[y/2])\mid y\in S \rightm$. We have just shown that $\chi_{d-1,S}^{(1)}(\mathbf{x})$ uniquely determines $D_{x_1}$. Hence, from the multiset $\leftm \chi_{d-1,S}^{(1)}(\mathbf{y}) \mid \mathbf{y}\in S^{d-1}\rightm$, one can compute the multiset $\leftm D_{y_1} \mid \mathbf{y} = (y_1, \ldots, y_{d-1})\in S^{d-1}\rightm$. But this multiset coincides with the multiset $\leftm D_{y} \mid y \in S\rightm$ except that all multiplicities are $|S|^{d-2}$ times larger.
\end{proof}

Now, we show that after the second iteration, we can restore the distance profile of $(b, x_1, \ldots, x_{d-1})$ for all $(x_1, \ldots, x_{d-1})\in S^{d-1}$.

\begin{Lemma}
    \label{l.profile}
For any $\mathbf{x} = (x_1, \ldots, x_{d-1}) \in S^{d-1}$, knowing $\chi_{{d-1},S}^{(2)}(\mathbf{x})$ and the multiset $\leftm \chi_{d-1,S}^{(1)}(\mathbf{y}) \mid \mathbf{y}\in S^{d-1}\rightm$, one can determine the distance profile of $(b, x_1, \ldots, x_{d-1})$ w.r.t.~$S$.
\end{Lemma}
\begin{proof}
   Since $d>2$, we have  
    \[\chi^{(2)}_{d-1,S}(\mathbf{x}) = \Big(\chi^{(1)}_{k,S}(\mathbf{x}), \leftm \left(\chi^{(1)}_{d-1,S}(\mathbf{x}[y/1]),\ldots, \chi^{(1)}_{d-1,S}(\mathbf{x}[y/k])\right)
    \mid y\in S\rightm\Big).\]
    From the tuple $\left(\chi^{(0)}_{d-1,S}(\mathbf{x}[y/1]),\ldots, \chi^{(0)}_{d-1,S}(\mathbf{x}[y/k])\right)$ we can restore the tuple of distances 
    $$(d(y, x_1), \ldots, d(y, x_{d-1})).$$ 
    In turn, by Lemma \ref{c.radius}, from  $\chi^{(1)}_{d-1,S}(\mathbf{x}[y/1])$ and $\leftm \chi_{d-1,S}^{(1)}(\mathbf{y}) \mid \mathbf{y}\in S^{d-1}\rightm$, we can restore $d(y, b)$. Hence, we can restore the whole distance profile of $(b, x_1, \ldots, x_{d-1})$.
\end{proof}

Finally, we show that knowing $\chi_{d-1,S}^{(3)}(\mathbf{x})$ for $\mathbf{x} = (x_1, \ldots, x_{d-1})$ and the multiset $\leftm \chi_{d-1,S}^{(1)}(\mathbf{y}) \mid \mathbf{y}\in S^{d-1}\rightm$, we can compute
\[\{EP(x_1, \ldots, x_{d-1}, y)\mid y\in S\}.\]

Since $d>2$, we have  
\[\chi^{(3)}_{d-1,S}(\mathbf{x}) = \Big(\chi^{(2)}_{d-1,S}(\mathbf{x}), \leftm \left(\chi^{(2)}_{d-1,S}(\mathbf{x}[y/1]),\ldots, \chi^{(2)}_{d-1,S}(\mathbf{x}[y/k])\right)
    \mid y\in S\rightm\Big).\]
    For every $y\in S$, knowing $\chi^{(2)}_{d-1,S}(\mathbf{x})$ and $\left(\chi^{(2)}_{d-1,S}(\mathbf{x}[y/1]),\ldots, \chi^{(2)}_{d-1,S}(\mathbf{x}[y/k])\right)$, we have to compute $EP(x_1, \ldots, x_{d-1}, y)$, that is, the distance matrix of $(b, x_1, \ldots, x_{d-1}, y)$ and the distance profiles of the tuples 
    \[
    (b, x_2, \ldots, x_{d-1}, y), \ldots, (x_1, \ldots, x_{d-2},b, y), (x_1, \ldots, x_{d-1}, b).\]
    Distance profiles can be computed by Lemma \ref{l.profile}.  The distances amongst elements of $\{x_1, \ldots, x_{d-1},y\}$ can be computed by definition from  $\chi^{(0)}_{d-1,S}(\mathbf{x})$ and $\left(\chi^{(0)}_{d-1,S}(\mathbf{x}[y/1]),\ldots, \chi^{(0)}_{d-1,S}(\mathbf{x}[y/k])\right)$. Distances to $b$ from these points can be computed by Lemma \ref{c.radius} from $\chi^{(1)}_{d-1,S}(\mathbf{x})$ and $\left(\chi^{(1)}_{d-1,S}(\mathbf{x}[y/1]),\ldots, \chi^{(1)}_{d-1,S}(\mathbf{x}[y/k])\right)$. 

Now that we have the enhanced profiles, we have to select one for which $(x_1,\dots,x_d)$ satisfies the cone condition. 
We first observe that, knowing $EP(x_1, \ldots, x_d)$, we can reconstruct $(b, x_1, \ldots, x_d)$  up to an isometry by Lemma \ref{l.anchor} (because inside $EP(x_1, \ldots, x_d)$ we are given the distance matrix of $(b, x_1, \ldots, x_d)$). This means that from $EP(x_1, \ldots, x_d)$ we can compute any function of $b, x_1, \ldots, x_d$ which is invariant under isometries. In particular, we can compute $\ad(b, x_1, \ldots, x_d)$. We will refer to $\ad(b, x_1, \ldots, x_d)$ as the  dimension of the corresponding enhanced profile.

We show that $\ad(S)$ is equal to the maximal  dimension of an enhanced profile. Indeed, first notice that $\ad(S) = \ad(\{b\}\cup S)$ because  $b$ is a convex combination of points of $S$. In turn, $\ad(\{b\}\cup S)$ is equal to the maximal $k$ for which one can choose $k$ points $x_1, \ldots, x_k\in S$ such that $x_1 - b, \ldots, x_k - b$ are linearly independent. Obviously, $k$ is bounded by the dimension of the space.  Hence, there will be an enhanced profile with the same maximal dimension $k$.

 If $\ad(S) < d$,  any enhanced profile with maximal dimension satisfies the initialization requirement, and we are done. Assume therefore that $\ad(S) = d$. Our task is to output some $EP(x_1, \ldots, x_{d})$ such that  
 \begin{enumerate}
     \item $\ad(b, x_1, \ldots, x_{d}) = d$, and
     \item there is no $x\in S$ such that $x - b$ belongs to the interior of $\cone(x_1 - b, \ldots, x_{d} - b)$. 
 \end{enumerate}
 
 For that, among all $d$-dimensional enhanced profiles, we output one which minimizes the solid angle at the origin, defined as 
\begin{equation}
    \label{int}
    \mathsf{Angle}(x_1-b,\dots,x_d-b)=\frac1d\mathsf{Vol}\left\{x\in\cone(x_1-b,\dots, x_d-b):\ \|x\|\le 1 \right\}
\end{equation}
(the solid angle is invariant under isometries, and hence can be computed from $EP(x_1, \ldots, x_d)$).

We have to show that for all $x\in S$ we have that $x-b$ lies outside the interior of $\cone(x_1 -  b, \ldots, x_{d} - b)$. To prove this, we need an extra lemma. We will say that a cone $C$ is \emph{simple} if $C=\cone(u_1, \ldots, u_{d})$, for some linearly independent $u_1, \ldots, u_{d}$. Observe that if $C = \cone(u_1, \ldots, u_{d})$ is a simple cone, then the interior of $C$ is the set  \[
int(C)=\{\lambda_1 u_1 +\ldots +\lambda_{d} u_{d}\mid \lambda_1, \ldots, \lambda_{d}\in (0,+\infty)\}.
\]
Note also that the boundary of $C$ consists of $d$ faces
\[F_i = C \cap \ls(u_1, \ldots, u_{i -1}, u_{i+1}, \ldots, u_{d}), \quad  i = 1, \ldots, d.
\]
\begin{Lemma}\label{l.volume}
    Let $C=\cone(\mathbf{u})\subseteq \mathbb R^d$ for $\mathbf u=(u_1,\dots,u_d)$ be a simple cone and let $y$ belong to the interior of $C$. Then  $\cone(\mathbf u[y/1])$ is a simple cone and $\mathsf{Angle}(\mathbf u[y/1]) < \mathsf{Angle}(\mathbf u)$.
\end{Lemma}

\begin{proof}
Since $y$ belongs to the interior of $\cone(\mathbf u)$, we have that 
\[y = \lambda_1 u_1 +\ldots + \lambda_d u_d,\]
for some $\lambda_1 > 0, \ldots, \lambda_d > 0$. The fact that $\lambda_1 > 0$ implies that $y, u_2, \ldots, u_d$ are linearly independent, and hence $\cone(\mathbf u[y/1])$ is a simple cone. Since $y\in \cone(\mathbf{u})$, we have that $\cone(\mathbf u[y/1])\subseteq \cone(\mathbf{u})$. Thus, to show that $\mathsf{Angle}(\mathbf u[y/1]) < \mathsf{Angle}(\mathbf u)$, it suffices to show that the volume of
\[\{x\in \cone(\mathbf u):\ \|x\| \le 1\}\setminus \{x\in \cone(\mathbf u[y/1]):\ \|x\| \le 1\}\]
is positive. We claim that  for any point $x = \mu_1 u_1 + \ldots +\mu_d u_d\in \cone(\mathbf u[y/1])$ we have $\mu_1 > 0 \implies \mu_2/\mu_1 \ge \lambda_2/\lambda_1$.  This is because $x$ can be written as a non-negative linear combination of $y, u_2, \ldots, u_d$. 
Since $\mu_1 > 0$, the coefficient in front of $y$ in this linear combination must be positive. Now, if the coefficient in front of $u_2$ is $0$, then the ratio between $\mu_2$ and $\mu_1$ is exactly as the ratio between $\lambda_2$ and $\lambda_1$, and if the coefficient before $u_2$ is positive, $\mu_2$ can only increase.

This means that no point of the form
\begin{equation}
\label{no_point}
x = \mu_1 u_1 + \ldots +\mu_d u_d, \qquad 0 < \mu_1,\,\, \mu_2/\mu_1 < \lambda_2/\lambda_1
\end{equation}
belongs to $\cone(\mathbf u[y/1])$. It remains to show that the set of points that satisfy \eqref{no_point} and lie in $\{x\in \cone(\mathbf u):\ \|x\| \le 1\}$ has positive volume.

Indeed, for any  $\varepsilon > 0$ and $\delta > 0$, consider a $d$-dimensional parallelepiped:
\[P = \{\mu_1 x_1 +\ldots + \mu_d x_d \mid  \mu_1 \in [\varepsilon/2, \varepsilon], \,\mu_2, \ldots, \mu_d\in[\delta/2, \delta]\}.\]
Regardless of $\varepsilon$ and $\delta$, we have that $P$ is a subset of $\cone(\mathbf{u})$ and its volume is positive. For all sufficiently small $\varepsilon, \delta$, we have that $P$ is a subset of the unit ball $\{x\in\mathbb{R}^d\mid \|x\| \le 1\}$. In turn, by choosing $\varepsilon$ to be sufficiently big compared to $\delta$, we ensure that all points of $P$ satisfy \eqref{no_point}.
\end{proof}

Now we can finish the proof of Proposition \ref{P.init}. Let $\mathbf x\in S^d$ minimize $\mathsf{Angle}(x_1-b,\dots,x_d-b)$ amongst $\mathbf x\in S^d$ such that $\ad(b,x_1,\dots,x_d)=d$. Assume that $\mathbf x$ does not satisfy the cone condition, and let $x\in S$ be such that $x - b\in Int(\cone(x_1 -  b, \ldots, x_{d} - b))$ holds. Then, by Lemma \ref{l.volume},  $\mathsf{Angle}(x -  b, x_2 - b \ldots, x_{d} - b)$ is strictly smaller than $\mathsf{Angle}(x_1 -  b, x_2 - b \ldots, x_{d} - b)$. As $\cone(x -  b, x_2 - b \ldots, x_{d} - b)$ is a simple cone, $x - b, x_2 - b, \ldots, x_{d} - b$ are linearly independent, and thus $\ad(b, x, x_2, \ldots, x_{d})=d$, contradicting the minimality hypothesis on $\mathbf x$, as desired.
\end{proof}

We now proceed with the reconstruction phase. 

\subsection{Reconstruction Algorithm} 

Assume an Initialization Data $(A,M_1,\dots,M_d)$ for a finite $S\subset \mathbb{R}^d$ is given. Our first task is to choose, up to isometry, positions for the points in the $(d+1)$-tuple $(b,x_1,\dots,x_d)$ corresponding to the matrix $A$.  We use the following classical lemma, whose proof is given e.g. in \cite[Sec. 2.2.1]{cox2001multidscaling}. 
\begin{Lemma}[Anchor Lemma]
\label{l.anchor}
     If $(u_1, \ldots, u_k) \in \mathbb{R}^d$ and $(v_1,\dots,v_k)\in \mathbb{R}^d$ have the same distance matrix, then there exists an isometry $f\colon\mathbb{R}^{d}\to\mathbb{R}^{d}$ such that $f(u_i) = v_i$ for all $i = 1, \ldots, k$.
\end{Lemma}

As for $d=2$, we put the barycenter of the cloud at the origin.
Then, we simply apply the Anchor Lemma to any collection of points $z_1,\dots,z_d \in \mathbb{R}^d$ such that our given $A$ is also the distance matrix of the tuple $(0,z_1,\dots,z_d)$. The Lemma then gives us an isometry $f\colon\mathbb{R}^{d}\to\mathbb{R}^{d}$ such that 
$f(b) = 0, f(x_1) = z_1,\ldots, f(x_{d}) = z_{d}$. 
As distance profiles are invariant under isometries, our given $M_i$ is also the distance profile of the tuple $(z_1, \ldots, z_{i - 1},0, z_{i+1}, \ldots, z_{d})$ w.r.t.~$f(S)$. Our task now is, from $M_1,\dots,M_d$, to uniquely determine the locations of all  points in $f(S)$. This would give us  $S$ up to an isometry. 
Since now we have locations for $(z_1,\dots,z_d)$,  we can in fact compute:    
\[
\ad(S)=\ad(b, x_1, \ldots, x_{d}) = \ad(0, z_1, \ldots, z_{d}) = \ld(z_1, \ldots, z_{d}), 
\]
where the first equality is guaranteed by the cone condition. The reconstruction algorithm depends on whether $\ad(S)=d$ or not. 

Consider first the case when $ \ad(S) =  \ld(z_1, \ldots, z_{d}) < d$. Then there exists $i \in \{ 1, \ldots, d\}$ such that $\ld(z_1, \ldots, z_{d}) = \ld(z_1, \ldots, z_{i - 1}, z_{i+1}, \ldots,  z_{d})$. This means that
$f(S)\subset\as( z_1, \ldots, z_{i-1},0, z_{i+1}, \ldots ,z_{d})$,  
since otherwise $f(S)$ would have larger affine dimension. 
We claim that, in this case, from $M_i$ we can restore the location of all points in $f(S)$. Indeed, from $M_i$ we know, for each $z \in f(S)$, a tuple with the distances from $z$ to $ z_1,\dots,z_{i-1},0,z_{i+1}, \dots, z_d$. 
As the next lemma shows, this information is enough to uniquely determine the location of $z$. 

\begin{Lemma}
\label{l.unique}
    Take any $x_1, \ldots, x_m\in\mathbb{R}^{d}$. Assume that $a, b\in \as(x_1, \ldots, x_m)$ are such that $d(a, x_i) = d(b, x_i)$ for all $i = 1, \ldots, m$. Then $a = b$.
\end{Lemma}
\begin{proof}
    We claim that for every $i = 2, \ldots, m$ we have $\langle a - x_1, x_i - x_1 \rangle = \langle b - x_1, x_i - x_1\rangle$. This is because $a-x_1$ and $b-x_1$ have the same distance to $x_i - x_1$ (which is equal to $d(a, x_i) = d(b, x_i)$) and, moreover, the norms of $a-x_1$ and $b-x_1$ coincide (and are equal to $d(a, x_1) = d(b, x_1)$). Hence, both $a - x_1$ and $b - x_1$ are solutions to the following linear system of equations:
    \begin{align*}
        \langle x, x_i - x_1\rangle = \langle a - x_1, x_i - x_1\rangle, \qquad i = 2, \ldots, m.
    \end{align*}
    This system has at most 1 solution over $x\in \ls(x_2 -x_1, \ldots, x_m - x_1)$. Moreover, $a - x_1$ and $b - x_1$ are both from $\ls(x_2 -x_1, \ldots, x_m - x_1)$ because $a, b\in \as(x_1, \ldots, x_m)$. Hence, $a - x_1 = b - x_1$, and $a = b$.
\end{proof}

It remains to reconstruct $f(S)$ when $\ad(S) = d$, in which case our pivot points $z_1, \ldots, z_{d}$ are linearly independent. Recall that $M_i$ is the distance profile of the tuple $(z_1, \ldots, z_{i-1},0, z_{i+1}, \ldots, z_{d})$ w.r.t.~$f(S)$. Moreover, since no $x$ in $S$ is such that $x-b$ lies in the interior of $\cone(x_1 - b, \ldots, x_{d} - b)$, we know that $f(S)$ must be disjoint from the interior of $\cone(z_1, \ldots, z_{d})$. As the next proposition shows, this information is enough to reconstruct $f(S)$ in this case as well, which finishes the proof of Theorem \ref{main} for $d>2$. 

\begin{Proposition}[The Reconstruction Lemma]\label{l.cone}
Assume that $z_1, \ldots, z_{d}\in\mathbb{R}^{d}$ are linearly independent. 
Let $T\subseteq \mathbb{R}^{d}$ be finite and disjoint from the interior of $\cone(z_1, \ldots, z_{d})$. If, for every $i = 1, \ldots, d$, we are given $z_i$ and also the distance profile of the tuple $(z_1, \ldots, z_{i-1},0, z_{i+1}, \ldots, z_{d})$ w.r.t.~$T$,  then we can uniquely determine $T$.
\end{Proposition}
\begin{proof}
  Let $P_i = \as(z_1, \ldots, z_{i -1}, 0, z_{i+1}, \ldots, z_{d})$. As the following lemma shows, knowing the distances from $s\in T$ to $z_1, \ldots, z_{i-1}, 0, z_{i+1}, z_{d}$, we can determine the position of $s$ uniquely up to the reflection through $P_i$.

    \begin{Lemma}[The Symmetric Lemma] \label{l.symmetric}
    Let $x_1, \ldots, x_m\in\mathbb{R}^{d}$ be such that $\as(x_1, \ldots, x_m)$ has dimension $d-1$.
    Assume that $a, b\in \mathbb{R}^{d}$ are such that $d(a, x_i) = d(b, x_i)$ for all $i = 1, \ldots, m$. Then either $a = b$ or $a$ and $b$ are symmetric w.r.t.~$\as(x_1, \ldots, x_m)$. 
\end{Lemma}
    \begin{proof}
       As in the proof of Lemma \ref{l.unique}, we have that $\langle a - x_1, x_i - x_1 \rangle = \langle b - x_1, x_i - x_1\rangle$ for every $i = 2, \ldots, m$. Consider orthogonal projections of $a - x_1$ and $b- x_1$ to $\ls(x_2 -x_1, \ldots, x_m - x_1)$. Both these projections are solutions to the system
        \begin{align*}
        \langle x, x_i - x_1\rangle = \langle a - x_1, x_i - x_1\rangle, \qquad i = 2, \ldots, m.
    \end{align*}
    This system has at most one solution over $x\in \ls(x_2 -x_1, \ldots, x_m - x_1)$. Hence, projections of $a - x_1$ and $b - x_1$ coincide. We also have that $\| a - x_1\| = \|b - x_1\|$, which implies that $a-x_1$ and $b - x_1$ have the same distance to $\ls(x_2 -x_1, \ldots, x_m - x_1)$. Since the dimension of $\ls(x_2 -x_1, \ldots, x_m - x_1)$ is $d - 1$, we get that either $a - x_1 = b - x_1$ or they can be obtained  from each other by the reflection through $\ls(x_2 - x_1, \ldots, x_m - x_1)$. After translating everything by $x_1$, we obtain the claim of the lemma.
    \end{proof}
In fact, if $s$ belongs to $P_i$, then there is just one possibility for $s$. Thus, we can restore all the points in $T$ that belong to the union $\bigcup_{i=1}^{d} P_i$. Let us remove these points from $T$ and update distance profiles by deleting the tuples of distances that correspond to the points that we have removed.

From now on, we may assume that $T$ is disjoint from $\bigcup_{i=1}^d P_i$. Hence, $T$ is also disjoint from the boundary of $C=\cone(x_1, \ldots, x_d)$, not only from its interior (every face of this cone lies on  $P_i$ for some $i$).

For $x\in \mathbb{R}^d$, we define $\rho(x) = \min_{i = 1, \ldots, d}\mathsf{dist}(x, P_i)$. Since $T$ is disjoint from $\bigcup_{i=1}^d P_i$, we have that $\rho(s) > 0$ for every $s\in T$. Moreover, from, say, the distance profile of $(0, z_2, \ldots, z_d)$, we can compute some $\varepsilon > 0$ such that $\rho(s)\ge\varepsilon$ for all $s\in T$. Indeed, recall that from the distance profile of $(0, z_2, \ldots, z_d)$, we get at most 2 potential positions for each point of $T$. This gives us a finite set $T^\prime$ (at most 2 times larger than $T$) which is a superset of $T$. Moreover, as $T$ is disjoint from  $\bigcup_{i=1}^d P_i$, we have that $T^\prime\setminus \bigcup_{i=1}^d P_i \supseteq T $
Thus, we can define $\varepsilon$ as the minimum of $\rho(x)$ over $T^\prime \setminus \bigcup_{i=1}^d P_i \supseteq T$. 

We conclude that $T$ is disjoint from 
\[A_0 = C \cup \{x\in\mathbb{R}^d \mid \rho(x) < \varepsilon\}\]
(moreover, the set $A_0$ is known to us).

 Our reconstruction procedure starts as follows. We go through all distance profiles, and through all tuples of distances in them. Each  tuple gives 2 candidates for a point in $T$ (that can be obtained from each other by the reflection through $P_i$). If one of the candidates lies in $A_0$, we know that we should take the other candidate. In this way, we may possibly uniquely determine some points in $T$. If so, we remove them from $T$ and update our distance profiles.

Which points of $T$ will be found in this way? Those that, for some $i$, fall into $A_0$ under the reflection through $P_i$. Indeed, these are precisely the points that give 2 candidates (when we go through the $i$th distance profile) one of which is in $A_0$. In other words, we will determine all the points that lie in $\bigcup_{i = 1}^{d} \mathsf{Refl}_i(A_0)$, where $\mathsf{Refl}_i$ denotes the reflection through $P_i$. After we remove these points, we know that the remaining $T$ is disjoint from $A_1 = A_0 \cup \bigcup_{i = 1}^{d} \mathsf{Refl}_i(A_0)$.

We then continue in exactly the same way, but with $A_1$ instead of $A_0$, and then with $A_2 =  A_1 \cup \bigcup_{i = 1}^{d} \mathsf{Refl}_i(A_1)$, and so on. It remains to show that all the points of $T$ will be recovered in this way. In other words, we have to argue that each point of $T$ belongs to some $A_i$, where 
\[
A_0 =  C \cup \{x\in\mathbb{R}^d \mid \rho(x) < \varepsilon\}, \qquad A_{i+1} =  A_i \cup \bigcup_{i = 1}^{d} \mathsf{Refl}_i(A_i).
\]
We will show this not only for points in $T$ but for all points in $\mathbb{R}^{d}$. Equivalently, we have to show that for every $x\in \mathbb{R}^d$ there exists a finite sequence of reflections $\tau_1, \ldots, \tau_k\in\{\mathsf{Refl}_1, \ldots, \mathsf{Refl}_{d}\}$ which brings $x$ inside $A_0$, that is, $\tau_k \circ\ldots \circ \tau_1(x) \in A_0$. 

We construct this sequence of reflections as follows.  Let $x$ be outside $A_0$. In particular, $x$ is outside the cone $C =  \cone(z_1, \ldots, z_{d})$. Then there exists a face of this cone such that $C$ is from one side of this face and $x$ is from the other side. Assume that this face belongs to the hyperplane $P_i$. We then reflect $x$ through $P_i$, and repeat this operation until we get inside $A_0$. 
We next show that the above process stops within a finite number of steps. For that, we introduce the quantity $\gamma(x)  = \langle x, z_1\rangle + \ldots +\langle x, z_{d}\rangle$. We claim that with each step, $\gamma(x)$ increases by at least $c\cdot\varepsilon$, where 
\[
c =2\min_{1\le i\le d}\mathsf{dist}(z_i,P_i).
\]
Note that $c > 0$ because, for every $i = 1, \ldots, d$, we have that $z_i\notin P_i$, by the linear independency of $z_1, \ldots, z_{d}$. Also observe that $\gamma(x)  \le |x|\sum_i|z_i|$ and reflections across the subspaces $P_i$ do not change $|x|$. Hence, $\gamma(x)$ cannot increase infinitely many times by some fixed positive amount. 

It remains to show that $\gamma(x)$ increases by at least $c \cdot \varepsilon$ at each step, as claimed. Note that reflection of $x$ across some $P_i$ does not change the scalar product of $x$ with those vectors among $z_1, \ldots, z_{d}$ that belong to $P_i$. The only scalar product that changes is $\langle x, z_i\rangle$, and the only direction which contributes to the change is the one orthogonal to $P_i$. 
Before the reflection, the contribution of this direction to the scalar product was $-d(x, P_i)\cdot d(z_i, P_i)$ (remember that $x$ and $z_i$ were from different sides of $P_i$ because $z_i\in C$). After the reflection, the contribution is the same, but with a positive sign. Thus, the scalar product increases by $2d(x, P_i)\cdot d(z_i, P_i)$. Now, we have $d(z_i, P_i) \ge c/2$ by definition of $c$ and $d(x, P_i)\ge \varepsilon$ if $x$ is not yet in $A_0$.
\end{proof}

\section{On the distinguishing power of one iteration of $d$-WL}
\label{sec:one_round}
\vspace{-0.2cm}
In this section, we discuss a somewhat different strategy to reconstruct $S$. It is clear that if for a point $z\in \mathbb{R}^d$ we are given the distances from it to  $d+1$ points in general position with known coordinates, then the position of $z$ is uniquely determined (see e.g.~Lemma \ref{l.unique}). Since $d$-WL colors $d$-tuples of points in $S$, a natural strategy to recover $S$ is to use the barycenter as an additional point. By Lemma \ref{l.orbits}, we know that distances to the barycenter from points of $S$ can be obtained after one iteration of $d$-WL. 
However, the information that allows us to match $d(z, b)$ to the distances from this $z$ to a $d$-tuple, is readily available only after two iterations of $d$-WL.  It follows that this simple strategy can be used to directly reconstruct $S$ from the second iteration of $d$-WL. We remark that this strategy is similar to the one used in \cite{3d-case} to uniquely determine $S$ up to isometries when the coloring we are initially given corresponds to certain Gram Matrices for $d$-tuples of points. 
Essentially, after one interaction of $d$-WL over this initial data, we obtain enough information to directly determine the location of each $z$ relative to a collection of $d+1$ points. In fact, it is not hard to show that from the first iteration of $d$-WL, applied to the distance graph of $S$, one can compute these Gram Matrices, thus providing an alternative proof that two iterations suffice for distinguishing geometric graphs.  

We show instead that only one iteration suffices. Our approach differs and depends on certain geometric principles that allow us to simplify the problem by conducting an exhaustive search across an exponentially large range of possibilities.

\begin{Theorem}\label{1-iteration}
    For any $d \ge 1$ and for any finite set $S \subseteq \mathbb{R}^d$, knowing the multiset $\leftm \chi_{d,S}^{(1)}(\mathbf{s}) | \mathbf{s} \in S^d \rightm$, 
    one can determine $S$ up to an isometry.
\end{Theorem}
\begin{proof}
    From $\chi_{d, S}^{(1)}(\mathbf{s})$, we can determine the tuple $\mathbf{s} = (s_1, \ldots, s_d)$ up to an isometry, since $\chi_{d, S}^{(1)}(\mathbf{s})$ gives us $\chi_{d, S}^{(0)}(\mathbf{s})$, which is the distance matrix of $\mathbf{s}$. In order to determine $S$, we consider two cases:
    
    \medskip {\bfseries Case 1}: \emph{For all $\mathbf{s}\in S^d$ it holds $\ad(\mathbf{s}) < d - 1$.} Then take $\chi_{d, S}^{(1)}(\mathbf{s})$ with maximal $\ad(\mathbf{s})$, and fix locations of points from $\mathbf{s}$ compatible with the distance specification, according to Lemma \ref{l.anchor}. All points of $S$ belong to $\as(\mathbf{s})$, otherwise we could increase $\ad(\mathbf{s})$. Indeed, since $\ad(\mathbf{s}) < d - 1$, we could throw away one of the points from the tuple without decreasing the dimension and replace it with a point outside $\as(\mathbf{s})$. We now can reconstruct the rest of $S$ uniquely up to an isometry. Indeed, in $\chi_{d, S}^{(1)}(\mathbf{s})$ we are given the multiset of $d$-tuples of distances to $\mathbf{s} = (s_1, \ldots, s_d)$ from the points of $S$, and it remains to use Lemma \ref{l.unique}.

    \medskip {\bfseries Case 2}: \emph{There are tuples with $\as(\mathbf{s}) = d - 1$.} We first observe that from the multiset $\leftm \chi_{d,S}^{(0)}(\mathbf{s}) \, | \, \mathbf{s} \in S^d \rightm$ we can compute the pairwise sum of distances between the points in $S$, i.e.,
    \[
    D_S = \sum_{x\in S}\sum_{y\in S} d(x, y).
    \]
    Indeed, from $\chi_{d,S}^{(0)}((s_1, \ldots, s_d))$, we determine $d(s_1, s_2)$. Hence, we can compute the sum: 
    \[
    \sum_{(s_1, \ldots, s_d)\in S^d} d(s_1, s_2) = D_S\cdot |S|^{d-2}.
    \]
    In our reconstruction of $S$, we go through all $\chi_{d, S}^{(1)}(\mathbf{s})$ with $\ad(\mathbf{s}) = d - 1$. For each of them, we fix positions of the points of the tuple $\mathbf{s}$ in any way that agrees with the distance matrix of this tuple. As before, $\chi_{d, S}^{(1)}(\mathbf{s})$ gives us the multiset of $d$-tuples of distances to $\mathbf{s}$ from the points of $S$. We call ``candidates for $S$ given $\mathbf{s}$" the set of point clouds $S'$ which have one point associated with each such $d$-tuple of distances, and realizing these distances to points in $\mathbf{s}$. We aim to find $\mathbf{s}$ for which, exactly one of these candidates can be isometric to $S$. We start with the following lemma: 
\begin{Lemma} \label{l.distinguish-halfspace}
    For any finite set $S\subseteq \mathbb{R}^d$ with $\ad(S)\ge d - 1$ there exist $x_1, \ldots, x_d\in S$ with $\ad(x_1, \ldots, x_d) = d - 1$ such that all points of $S$ belong to the same half-space with respect to the  hyperplane $\as(x_1, \ldots, x_d)$.
\end{Lemma}
\begin{proof}The general idea of the proof is the following.
    If $\ad(S)=d-1$ then the extreme points of the convex hull of $S$ contain an affinely independent set of cardinality $d$, which then gives the desired $\mathbf{s}$. The half-space condition in the lemma is automatically verified in this case. If $\ad(S)=d$ then to find $\mathbf{s}$ we can proceed by moving a $(d-1)$-plane from infinity towards $S$ until it touches $S$ in at least one point, then iteratively we rotate the plane around the subspace containing the already touched points of $S$, until a new point in $S$ prohibits to continue that rotation. We stop within at most $d$ iterations, when no further rotation is allowed, in which case the plane has an affinely independent subset in common with $S$.

    Formally, we need to find a hyperplane $H$ such that, first, all points of $S$ belong to the same half-space w.r.t.~$H$, and second, $\ad(H\cap S) = d - 1$.

    To start, we need to find a hyperplane $H$ such that, first, all points of $S$ belong to the same half-space w.r.t.~$H$, and second, $H\cap S\neq \varnothing$. For instance, take any non-zero vector $\alpha \in\mathbb{R}^d$, consider $m = \max_{x\in S} \langle \alpha, x\rangle$ and define $H$ by the equation $\langle \alpha, x\rangle = m$. Now, take any $x_1\in H \cap S$.  After translating $S$ by $-x_1$, we may assume that $x_1 = 0$. 

Now, among all hyperplanes $H$ that contain $x_1  = 0$ and satisfy the  condition that all points of $S$ lie in the same half-space w.r.t.~$H$, we take one that contains most points of $S$. We claim that $\ad(H\cap S) = d - 1$ for this $H$. Assume for contradiction that $\ad(H\cap S) < d - 1$. Define $U = \as(H\cap S)$. Since $H\cap S$ contains $x_1 = 0$, we have that $U\subseteq H$ is a linear subspace, and its dimension is less than $d - 1$. Hence, since $\ad(S) \ge d - 1$, there exists $x_2\in S\setminus U$. Note that $x_2\notin H$ because otherwise $x_2$ belongs to $H\cap S \subseteq U$.

Let $\alpha$ be the normal vector to $H$. Since all points of $S$ lie in the same half-space w.r.t.~$H$, w.l.o.g.~we may assume that $\langle \alpha, s\rangle \ge 0$ for all $s\in S$. In particular, $\langle \alpha, x_2\rangle > 0$ because $x_2\notin H$.

Let $U^\bot$ denote the orthogonal complement to $U$. Since $\alpha$ is the normal vector to $H\supseteq U$, we have that $\alpha\in U^\bot$. We need to find some $\beta\in U^\bot$  
which is not a multiple of $\alpha$ but satisfies $\langle \beta, x_2\rangle > 0$. 
Indeed, the dimension of $U$ is at most $d - 2$, and hence the dimension of $U^\bot$   is at least 2. Now, since $\langle\alpha, x_2\rangle > 0$, we can take any  $\beta\in U^\bot$ which is sufficiently close to $\alpha$.

For any $\lambda \ge 0$, let $H_\lambda$ be the hyperplane, defined by $\langle \alpha - \lambda \beta, x\rangle = 0$ (this is a hyperplane and not the whole space because $\beta$ is not a multiple of $\alpha$). We claim that for some $\lambda > 0$, we have that $H_\lambda$ has more points of $S$ than $H$ while still all points of $S$ lie in the same half-space w.r.t.~$H_\lambda$. This would be a contradiction.

Indeed, define $S_\beta = \{s\in S\mid \langle s, \beta\rangle > 0\}$. Note that $S_\beta$, by definition of $\beta$, contains $x_2$ and hence is non-empty. Moreover, $S_\beta$ is disjoint from $H\cap S$. This is because $H\cap S\subseteq U$ and $\beta\in U^\bot$.

Define
\[\lambda = \min_{s\in S_\beta} \frac{\langle \alpha, s\rangle }{\langle \beta, s\rangle}\]
First, $H_\lambda\supseteq U\supseteq H\cap S$ because $\alpha - \lambda\beta\in U^\bot$. Moreover, $H_\lambda$ contains at least one point of $S$ which is not in $H$. Namely, it $H_\lambda$ contains any $s\in S_\beta$, establishing the minimum in the definition of $\lambda$ (and recall that $S_\beta$ is disjoint from $H\cap S$). 
 Indeed, for this $s$ we have $\lambda = \frac{\langle \alpha, s\rangle }{\langle \beta, s\rangle}$. Hence, $\langle \alpha, s\rangle - \lambda \langle \beta, s\rangle = 0 = \langle\alpha - \lambda b, s\rangle \implies s\in H_\lambda$.
 
It remains to show that all points of $S$ lie in the same half-space w.r.t.~$H_\lambda$. More specifically, we will show that $\langle \alpha - \lambda \beta, s\rangle \ge 0$ for all $s\in S$. First, assume that $\langle s, \beta\rangle = 0$. Then $\langle \alpha - \lambda \beta, s\rangle = \langle \alpha, s\rangle \ge 0$ because all points of $S$ lie in the ``non-negative'' half-space w.r.t.~$\alpha$. Second, assume that $\langle s, \beta\rangle > 0$. Then $s\in S_\beta$. Hence, by definition of $\lambda$, we have $\lambda \le \frac{\langle \alpha, s\rangle }{\langle \beta, s\rangle}$. This means that $\langle \alpha - \lambda \beta, s\rangle = \langle \alpha, s\rangle - \lambda \langle \beta, s\rangle\ge 0$, as required.
\end{proof}
Next, consider the following simple geometric observation:
\begin{Lemma}\label{l.geometricobs} 
    Let $P\in \mathbb R^d$ be a hyperplane and consider two points $a,b\in \mathbb R^d\setminus P$ that lie in same half-space w.r.t. $P$. Let $a',b'$ be the reflections of $a,b$ through $P$. Then $d(a',b')=d(a,b)<d(a,b')$.
\end{Lemma}
\begin{proof}
    It suffices to restrict to the plane $\as(\{a,b,a'\})$, and thus we take $d=2$ and up to isometry we may fix $P$ to be the $x$-axis, $a=(0,y)$, $b=(x,y')$, $b'=(x,-y')$, with $y,y'>0$. Then it follows that $d(a,b)^2=x^2 + (y-y')^2<x^2+(y+y')^2=d(a,b')^2$. As reflections are isometries, $d(a,b)=d(a',b')$.
\end{proof}
{\bfseries Claim:} If $\mathbf{s}$ is as in Lemma \ref{l.distinguish-halfspace}, then we have the following:
\begin{itemize}
    \item Exactly one of the candidates for $S$ given $\mathbf{s}$, up to reflection across $\as(\mathbf{s})$, is completely contained in one of the half-spaces determined by $\as(\mathbf{s})$.
    \item A candidate $S'$ as in the previous point is the only one of the candidates for $S$ given $\mathbf{s}$, up to reflection across $\as(\mathbf{s})$, for which $D_{S'}=D_S$.
\end{itemize} 
    To prove the first item, we use only the property that $\ad(\mathbf{s})=d-1$, with which by Lemma \ref{l.symmetric}, each point of a candidate for $S$ given $\mathbf{s}$, has either two possible locations (related by a reflection across $\as(\mathbf{s})$) or a single possible location if it belongs to $\as(\mathbf{s})$. For the second item, let $S'$ be as above and let $S''$ be a candidate for $S$ given $\mathbf{s}$, which is not completely contained in one of the halfspaces determined by $\as(\mathbf{s})$. We now consider each term $d(x',y')$ in the sum defining $D_{S'}$, comparing the corresponding term $d(x'',y'')$ from $D_{S''}$, where $x''=x'$ or is a reflection across $\mathbf{s}$ of $x'$ and similarly for $y''$ and $y'$. By Lemma \ref{l.geometricobs}, either $d(x'',y'')=d(x',y')$ in case $x'',y''$ are in the same half-space determined by $\as(\mathbf{s})$, or $d(x'',y'')>d(x',y')$ otherwise. Summing all terms, by the property of $S',S''$ we find $D_{S'}<D_{S''}$. By the same reasoning with $S$ instead of $S''$, since we are assuming that $\mathbf{s}$ satisfies Lemma \ref{l.geometricobs}, we have $D_S=D_{S'}$, completing the proof of the second item and of the claim.

    The reconstruction of $S$ in Case 2, can therefore be done as follows: we run through all $\mathbf{s}$ such that $\ad(\mathbf{s})=d-1$, and for each such $\mathbf{s}$ we construct all candidates $\widetilde S$ for $S$ given $\mathbf{s}$, and calculate $D_{\widetilde S}$ for each of them. Lemma \ref{l.distinguish-halfspace} guarantees that we run into some  $\mathbf{s}$ for which, up to isometry, only one such $\widetilde{S}$ realizes $D_{\widetilde S}=D_S$. This is our unique reconstruction of $S$. 
\end{proof}

\section{1-WL does not recognize planarity in $\mathbb{R}^3$. }
\label{sec:1wlcounterexample}
\begin{Proposition}
There exist two point clouds in $\mathbb{R}^3$, one of which is planar and the other is not, that are not distinguished by the 1-WL test.
\end{Proposition}
\begin{proof}
One point cloud is a regular hexagon with side of length 1. Every vertex of it has two adjacent vertices at distance 1, another 2 vertices at distance $\sqrt{3}$, and the opposite vertex at distance $2$.
\begin{center}
\begin{tikzpicture}
   \node[inner sep=1pt,circle,draw,fill] (1) at (1.73,0) {};
   \node[inner sep=1pt,circle,draw,fill] (2) at (0.86, 1.5) {};
   \node[inner sep=1pt,circle,draw,fill] (3) at (-0.86, 1.5) {};

   \node[inner sep=1pt,circle,draw,fill] (4) at (-1.73,0) {};

   \node[inner sep=1pt,circle,draw,fill] (5) at (-0.86, -1.5) {};
   \node[inner sep=1pt,circle,draw,fill] (6) at (0.86, -1.5) {};

\draw (1) -- (2) node [midway, above, fill=none] {\small $1$};

\draw (1) -- (6) node [midway,below,  fill=none] {\small $1$};

\draw (1) -- (3) node [midway, above, fill=none] {\small $\sqrt{3}$};

\draw (1) -- (5) node [midway, below,  fill=none] {\small$\sqrt{3}$};

\draw (1) -- (4) node [midway, below,  fill=none] {\small$2$};
   \draw (2) -- (3);
   \draw (3) -- (4);
   \draw (4) -- (5);
      \draw (5) -- (6);
\end{tikzpicture}
\end{center}

We construct a non-planar point cloud with 6 nodes, where, for every point, the multiset of distances to other points is also $\leftm1,1,\sqrt{3}, \sqrt{3}, 2\rightm$. Any such point cloud is not distinguished from the regular hexagon as above by the 1-WL test.

We start from  a regular hexagon with side of length $1/\sqrt{3}$, belonging to some plane. We split its vertices into 2 groups of 3 vertices, where in each group vertices are pairwise non-adjacent. We colour vertices of one group into green and of the other into red.

\begin{center}
\begin{tikzpicture}
      \node[inner sep=2pt,circle,draw,fill=green] (1a) at (5+1,0) {};
   \node[inner sep=2pt,circle,draw,fill=red] (2a) at (5+0.5, 0.86) {};
   \node[inner sep=2pt,circle,draw,fill=green] (3a) at (5-0.5, 0.86) {};

   \node[inner sep=2pt,circle,draw,fill=red] (4a) at (5-1,0) {};

   \node[inner sep=2pt,circle,draw,fill=green] (5a) at (5-0.5, -0.86) {};
   \node[inner sep=2pt,circle,draw,fill=red] (6a) at (5+0.5, -0.86) {};

\draw (1a) -- (6a);
   \draw (2a) -- (3a) node [midway, below,  fill=none] {\tiny$\frac{1}{\sqrt{3}}$};
   \draw (3a) -- (4a);
   \draw (4a) -- (5a);
      \draw (5a) -- (6a);
      \draw (1a) -- (2a);


 \node[inner sep=2pt,circle,draw,fill=green] (1b) at (6,0.5 +0.5) {};
   \node[inner sep=2pt,circle,draw,fill=green] (3b) at (5-0.5, 0.5 +1.36) {};
   \node[inner sep=2pt,circle,draw,fill=green] (5b) at (5-0.5, 0.5 +-0.36) {};

   \node[inner sep=2pt,circle,draw,fill=red,fill opacity=0.3] (2b) at (5+0.5, 0.36-0.5) {};
   \node[inner sep=2pt,circle,draw,fill=red,fill opacity=1] (4b) at (5-1,-0.5-0.5) {};
   \node[inner sep=2pt,circle,draw,fill=red,fill opacity=1] (6b) at (5+0.5, -1.36-0.5) {};

      \draw[dashed] (1a) -- (1b);
      \draw[dashed] (2a) -- (2b);
      \draw[dashed] (3a) -- (3b);
      \draw[dashed] (4a) -- (4b);
      \draw[dashed] (5a) -- (5b);
      \draw[dashed] (6a) -- (6b);

 \draw [decorate,decoration={brace,amplitude=4pt,mirror},xshift=0.5cm,yshift=0pt]
      (5.6,0) -- (5.6,0.5+0.5) node [midway,right,xshift=.1cm] {\tiny$\sqrt{2/3}$};    



\end{tikzpicture}
\end{center}

We lift green vertices above the plain, moving perpendicularly to it, by distance $h = \sqrt{2/3}$. Similarly, we push red vertices down, below the plane, by the same distance $h = \sqrt{2/3}$, again moving perpendicularly to the plain. 

Take any vertex. Distances from it to two other vertices of the same color do not change when we move them up or down. In the initial hexagon, vertices of the same color are at distance $\sqrt{3}$ times bigger than the size of the hexagon, 
In our case, this gives distance $  \sqrt{3} \cdot (1/\sqrt{3}) = 1$. Now, distances to vertices of the opposite colour can be computed by the Pythagorean theorem. Namely, the square of that distance is the square of the initial distance in the hexagon plus $(2h)^2 = 8/3$. This gives two following two numbers:
\[\sqrt{(1/\sqrt{3})^2 + 8/3} = \sqrt{3},\]
\[\sqrt{(2/\sqrt{3})^2 + 8/3} = 2,\]
as required.
\end{proof}

\newpage
\section{2-FWL is not complete in $\mathbb{R}^6$}
\label{sec:2wl}
\begin{Theorem}
There exist two non-isometric point-clouds in $\mathbb{R}^6$ that cannot be distinguished by 2-WL.
\end{Theorem}
\begin{proof}
We use an example from~\cite{DBLP:journals/corr/abs-2112-09992} of two non-isomorphic graphs, not distinguishable\footnote{The 2-WL test for graphs is defined exactly as for points clouds, except that the initial coloring of a pair of nodes $(u, v)$ indicates, whether $u$ and $v$ are equal and whether they are connected by an edge.} by 2-WL . One of them is $L(K_{4,4})$, the line graph of the complete bipartite graph $K_{4,4}$. It can be described as follows: its set of nodes is $\{0, 1, 2, 3\}^2$ (encoding edges of $K_{4,4}$), and we connect two distinct pairs $(a, b)$ and $(c, d)$ by an edge if and only if $a = c$ or $b = d$.

Remarkably, we do not need an explicit description of the second graph. We just need to know that there exists a graph, not isomorphic to $L(K_{4,4})$ but not distinguishable from it by the 2-WL test.

For $a, b, c, \in\mathbb{R}$, we define the \emph{$(a, b, c)$-filling} of a simple undirected graph $G = (V, E)$  as the symmetric matrix $A\in\mathbb{R}^{V\times V}$ given by:
\[A_{uv} = \begin{cases}a & \text{if } u = v, \\ b & \text{if } u \neq v \text{ and } \{u, v\}\in E,\\c &\text{if } u \neq v \text{ and } \{u, v\}\notin E, \end{cases}\qquad u, v\in V.\]

We start by observing that $L(K_{4,4})$ can be ``embedded'' into $\mathbb{R}^6$ in the following sense. There is a mapping from the set of nodes of $L(K_{4,4})$ to $\mathbb{R}^6$ such that adjacent nodes are mapped to points at distance 1, and non-adjacent nodes are mapped to points at distance $\sqrt{2}$. Equivalently, we show that the $(0, 1, \sqrt{2})$-filling of  $L(K_{4,4})$ is the distance matrix of some point cloud $C\subseteq\mathbb{R}^6$.

Namely, we take two regular 3-dimensional tetrahedrons with side of length 1, one in the first 3 coordinates of $\mathbb{R}^6$, and the other in the last 3 coordinates of $\mathbb{R}^6$.  Denoting set of vertices of the first tetrahedron by $\{x_0, x_1, x_2, x_3\}$, and the set of vertices of the second tetrahedron by $\{y_0, y_1, y_2, y_3\}$,  we consider the following embedding of $L(K_{4,4})$ to $\mathbb{R}^6$:
\[(i, j) \mapsto (x_i, y_j), \qquad i, j\in\{0, 1, 2, 3\}.\]
By construction, the distance between $x_i$ and $x_j$ is 1 for $i\neq j$, as well as between $y_i$ and $y_j$. Hence, adjacent pairs of nodes of $L(K_{4,4})$ (corresponding to pairs, differing in exactly one position) are mapped to points at distance 1. Likewise, non-adjacent  pairs of nodes (corresponding to pairs that differ in both positions), by Pythagorean theorem, are mapped to points at distance $\sqrt{2}$.

It is now enough to demonstrate that for   
any graph $G$ which is 2-WL-equivalent to $L(K_{4,4})$, it also holds that its $(0, 1,\sqrt{2})$-filling is the distance matrix of some point could $C^\prime\subseteq\mathbb{R}^6$. Indeed, we then can take $G$, not isomorphic to  $L(K_{4,4})$ but 2-WL-equivalent to it. The  point clouds $C, C^\prime\subseteq \mathbb{R}^6$, corresponding to $(0, 1,\sqrt{2})$-fillings of $L(K_{4,4})$ and $G$ will be isometric but not distinguishable by 2-WL. 

Assume that two tetrahedrons in the construction of the embedding of $L(K_{4,4})$ have their center at the origin. Then all points of $C = \{(x_i, y_j) \mid i,j\in\{0, 1, 2, 3\}\}$ are at the same distance $a = \sqrt{2}\cdot d$ from the origin, where $d$ is the distance from the center of a regular tetrahedron with side 1 to its vertices.

We will consider only point clouds in $\mathbb{R}^6$ where all points are at distance $a$ from the origin. Under this assumption, the distance between two points  is uniquely determined by their scalar product, and vice versa. This allows us to work with Gram matrices instead of distance matrices in our argument. 

More specifically, since the distance matrix of $C$ is the $(0, 1, \sqrt{2})$-filling of $L(K_{4,4})$, we also have that the Gram matrix of $C$ is the $(a^2, b, c)$-filling of $L(K_{4,4})$ for some $b, c\in\mathbb{R}$. Now, take any $G$ which is 2-WL equivalent to $L(K_{4,4})$. Instead of showing that the $(0, 1, \sqrt{2})$-filling of $G$ is the distance matrix of some point cloud $C^\prime\subseteq \mathbb{R}^6$, it is enough to show that the $(a^2, b, c)$-filling of $G$  is the Gram matrix of some point cloud $C^\prime\subseteq \mathbb{R}^6$.

We rely on the following lemma.
\begin{Lemma}
\label{lem_spectral}
Let $G_1$ and $G_2$ be two 2-WL equivalent simple undirected graphs. Then for any $a, b, c \in\mathbb{R}$, the $(a, b, c)$-fillings of $G_1$ and $G_2$ have the same spectra.
\end{Lemma}
We apply it to $M$, the Gram matrix of $C$, and $M^\prime$, the $(a^2, b, c)$-filling of $G$. Due to  the lemma, they have the same spectra.  Since $M$ is the Gram matrix of a 6-dimensional point cloud with $16$ points, we have
\[M= X^T X\] 
for some $X\in\mathbb{R}^{6\times 16}$. Therefore, $M$ is positive semi-definite and its rank is at most $6$. Therefore, all eigenvalues of $M$ are non-negative, and at most $6$ of them are  positive. By Lemma \ref{lem_spectral}, the same is true for $M^\prime$. In particular, for some orthogonal matrix $Y$, we have that $Y^T M^\prime Y$ is a diagonal matrix, with 6 non-negative numbers and 10 zeros on the diagonal.  Hence, $Y^T M^\prime Y$ is the Gram matrix of some point cloud of size 16 in $\mathbb{R}^6$. Namely, this is the point cloud, consisting of $6$ basic vectors, multiplied by square roots of 6 largest eigenvalues of $M^\prime$, together with 10 repetitions of the zero vector. This allows us to write:
\[Y^T M^\prime Y = Z^T Z\]
for some $Z\in\mathbb{R}^{6\times 16}$. By orthogonality of $Y$, we can now write:
\[M^\prime = (ZY^T)^T Z Y^T,\]
which means that $M^\prime$ is the gram matrix of the set of columns of $Z Y^T$, and columns of this matrix are $6$-dimensional. 
\begin{proof}[Proof of Lemma \ref{lem_spectral}]
We first observe the following: for any graph $G = (V, E)$, and for any $a, b, c \in\mathbb{R}$, if $A$ is the $(a, b, c)$-filling of $G$, then $\chi^{(t)}(u, v)$, the 2-WL label of $(u, v)\in V^2$ after $t$ iterations, uniquely determines $A^{2^t}_{u,v}$. We establish this by induction on $t$. For $t = 0$, we have that $\chi^{(t)}(u, v)$ uniquely determines $A_{u,v}$ because it contains information, whether $u = v$ and whether $u$ and $v$ are connected by an edge (and the value of $A_{u,v}$ is defined by this information). We now perform the induction step. By definition, $\chi^{(t+ 1)}(u, v)$ uniquely determines the multiset $\leftm(\chi^{(t)}(w, v), \chi^{(t)}(u, w)) \mid w\in V\rightm$, which, by the induciton hypothesis, determines the sum
\[\sum\limits_{w\in V} A^{2^t}_{w,v} A^{2^t}_{u,w} = A^{2^{t + 1}}_{u,w},\]
as required. 

This observation implies that the multiset of 2-WL labels of $G$ after $t$ iterations uniquely determines the trace of $A^{2^t}$ (we have to note here that from this multiset, we can extract the multiset $\leftm \chi^{(t)}(v, v) \mid v\in V\rightm$ because the 2-WL label of a pair of nodes determines, whether these two nodes are equal).

We now let $A_1$ and $A_2$ be the $(a, b, c)$-fillings of two 2-WL equivalent graphs $G_1$ and $G_2$. By the remark from the previous paragraph, $A_1^{2^t}$ and $A_2^{2^t}$ have the same trace for any $t\ge 0$. If $\lambda_{11}\ge \lambda_{12}\ge\ldots  \ge\lambda_{1n}$ and $\lambda_{21}\ge \lambda_{22}\ge\ldots  \ge\lambda_{2n}$ are the spectra of $A_1$ and $A_2$, respectively, we obtain:
\begin{equation}
\label{eq_lambda}
\lambda^{2^t}_{11} + \lambda^{2^t}_{12} + \ldots + \lambda^{2^t}_{1n} = \lambda^{2^t}_{21} + \lambda^{2^t}_{22} + \ldots + \lambda^{2^t}_{2n}
\end{equation}
for every $t \ge 0$. Without loss of generality, we may assume that all eigenvalues are positive (if not, add a large positive number $\Delta$ to $a$, the same $\Delta$ will be added to all eigenvalues of both matrices). We establish that $\lambda_{1i}= \lambda_{2i}$ for all $i\in\{1, 2,\ldots, n\}$. Indeed, if not, take the smallest $i$ with  $\lambda_{1i}\neq \lambda_{2i}$. Cancelling out the first $i-1$ equal terms, we obtain that one of the sides of \eqref{eq_lambda} grows faster than the other as $t\to\infty$, a contradiction.
\end{proof}
\end{proof}

\section{Conclusion and open problems}

For any positive integer $d$, one can define $k_d$ as the minimal positive integer $k$ such that the geometric $k$-WL test is complete for $\mathbb{R}^d$. Our main result is an upper bound $k_d \le d -1$ for all $d\ge 2$. For $d = 2, 3$, the exact values of $k_d$ are $k_2 = 1, k_3 = 2$. It is  open if $k_4 = 2$ or $k_5 = 2$, and we have shown that $k_6 > 2$. More generally, it is interesting to find out the rate of growth of $k_d$ as a function of $d$. It can be shown that $k_d = \Omega(d)$~\cite{test1}, and we conjecture that $k_d = d + O(1)$.

Our results imply that standard MPNNs, equivalent to the 1-WL test, are, in principle,  capable of learning all geometric information in planar point clouds, in just 3 iterations. It has to be noted, however, that the update time of MPNNs is quadratic (in the number of points) in this case as they are used over complete distance graphs. This puts more challenges for MPNNs in this context than in their standard applications, where they are run over relatively sparse graphs. The situation becomes more complicated in the 3-dimensional case, with the update time of 2WL-like GNNs becoming cubic. This questions scalability of such architectures for applications like learning properties of long molecules. A major open problem is to come up with an architecture for $\mathbb{R}^3$ with faster update time without sacrificing completeness.

\paragraph{Acknowledgments}
Barcel\'{o}, Kozachinskiy, Petrache, and Rojas are funded by the National Center for Artificial Intelligence CENIA
FB210017, Basal ANID. Barcel\'{o} is funded by ANID Millennium Science Initiative Program ICN17002. Kozachinskiy
is funded by ANID Fondecyt Iniciaci\'{o}n 11250060.


\end{document}